\documentclass[10pt,journal,compsoc]{IEEEtran}

\ifCLASSOPTIONcompsoc
  \usepackage[nocompress]{cite}
\else
  \usepackage{cite}
\fi

\usepackage{amsmath,amssymb,amsthm}

\usepackage{microtype}
\usepackage{graphicx}
\usepackage{booktabs} % for professional tables
\usepackage[mathscr]{euscript}
\usepackage{algorithm2e}
\usepackage[pagebackref]{hyperref}
\usepackage{placeins}
\usepackage{float}
\usepackage{mathtools}
\usepackage{listings}
\usepackage{xcolor}
\usepackage{graphicx}
\usepackage{cleveref}

\newcommand{\NAN}{n/a}

\newtheorem{theorem}{Theorem}[section]

\newtheorem{corollary}[theorem]{Corollary}

\newcommand{\masknet}{\Psi}
\newcommand{\frozen}{\Phi}

\newcommand{\unet}{U-Net}
\newcommand{\simm}{\ensuremath{\operatorname{\delta}_\frozen}}
\newcommand{\mask}{\ensuremath{\operatorname{\omega}}}
\newcommand{\maskins}{\ensuremath{\operatorname{\omega}_{\text{ins}}}}
\newcommand{\maskdel}{\ensuremath{\operatorname{\omega}_{\text{del}}}}
\newcommand{\maskrand}{\ensuremath{\operatorname{\omega}_{\text{rnd}}}}
\newcommand{\I}{\ensuremath{\mathcal{I}}}
\renewcommand{\I}{\ensuremath{[d]}}
\newcommand{\J}{\mathcal{J}}
\newcommand{\E}{\mathbb{E}}
\newcommand{\ic}{\ensuremath{c_\text{ins}}}
\newcommand{\dc}{\ensuremath{c_\text{del}}}
\newcommand{\w}{w}
\newcommand{\outdim}{c}
\newcommand{\fa}{a}
\newcommand{\fr}{\sigma_a}
\newcommand{\frInv}{\sigma^{-1}_a}
\newcommand{\loss}{\mathcal{L}}
\newcommand{\pow}{\mathscr{P}}

\newcommand{\gum}{g}

\newcommand{\ranem}{Ra-NEM}
\newcommand{\rise}{RISE}
\newcommand{\gradcam}{Grad-CAM++}
\newcommand{\intgrad}{Integrated Gradients}
\newcommand{\gradshap}{Grad-SHAP}
\newcommand{\spm}{Smooth Pixel Mask}
\newcommand{\saliency}{Saliency}
\newcommand{\nemt}{NEMt}
\newcommand{\iba}{IBA}

\lstdefinestyle{python}{
    language=Python,
    basicstyle=\ttfamily\footnotesize,     % Font style and size
    keywordstyle=\bfseries\color{blue},   % Keywords in bold blue
    commentstyle=\itshape\color{green!50!black}, % Comments in italic green
    stringstyle=\color{orange},           % Strings in orange
    showstringspaces=false,               % Don't show spaces in strings
    breaklines=true,                      % Automatic line breaking
    frame=none,                         % Frame around the code
    numbers=left,                         % Line numbers on the left
    numberstyle=\tiny\color{gray},        % Style of the line numbers
    tabsize=4,                            % Tab size
    morekeywords={logsumexp, detach,zeros_like,sort,log,rand_like}, 
}

\usepackage{hyperref}

\begin{document}

\title{For Those Who Believe in Faithfulness: 
Optimizing the Area Under Insertion and Deletion Curves for Ranking Relative Feature Importance} 
\author{Bjørn Leth Møller, Bulat Ibragimov, Christian Igel
% <-this % stops a space
\IEEEcompsocitemizethanks{\IEEEcompsocthanksitem B. Leth Møller was and B. Ibragimov and C. Igel are with the Department
of Computer Science, University of Copenhagen, Denmark.\protect\\
E-mail: \{bjm, bulat, igel\}@di.ku.dk}}

\IEEEtitleabstractindextext{%
\begin{abstract}
  The adoption of machine learning for socially relevant tasks requires effective explainable artificial intelligence (XAI) methods to better understand the behavior of machine learning models. 
Attribution methods are a popular XAI approach in which input-output relationships are characterized by heat maps that reflect the relative importance of input features for a particular prediction. 
The quality of such maps is often assessed by measuring faithfulness  based on the area under  insertion and deletion curves, which measures changes in the model output as features are
added and removed. In this study, we derive an objective function from this notion of faithfulness and a way to approximate its gradient.
We establish the connection between insertion curves and top-$k$ feature selection, which leads to a loss function measuring the quality of attributions. Randomization of the loss allows us to efficiently approximate its gradient. To show the effectiveness of the general approach, we combine the loss function with the neural explanation mask framework.
The resulting method, termed Ra-NEM, can be used with any differentiable model without affecting the model's performance. 
Experiments demonstrate that Ra-NEM provides accurate attributions robustly and efficiently. Compared to other algorithms, 
the attributions have not only higher faithfulness but also perform well in terms of other XAI metrics. The high inference speed of Ra-NEM makes the method suitable for online applications. \href{https://github.com/baerminator/Ra_Nem}{The code is available online.}
%\emph{Code is available as supplementary material.}

\end{abstract}
}

\maketitle
% <bulat@di.ku.dk> <bjm@di.ku.dk>

\section{Introduction}
\label{sec:introduction}
% General introduction
Explainable AI (XAI) aims to clarify how and why a model produces specific outputs from given inputs. This insight is needed to provide a higher level of transparency and safety in real-world applications \cite{gerlings2020reviewing} and is increasingly emphasized in regulatory frameworks \cite{chung2024false}. 
A wide range of tasks fall under the umbrella of XAI. One of the most fundamental ones is to estimate the importance of individual input features for the model prediction, usually through a feature map, also referred to as an \textit{attribution}, which assigns to each input feature some importance score. While verification of attribution maps is still an open research question, one quality commonly measured during evaluation is \textit{faithfulness}, which attempts to quantify how well the attribution aligns with the model's behavior for a given input-output pair. 
Faithfulness, in turn, is often assessed through \textit{deletion} and \textit{insertion curves} \cite{petsiuk2018rise,pmlr-v235-muzellec24a,fong2017interpretable}. Introduced in \cite{petsiuk2018rise}, the insertion curve measures changes in the model output as features are added in order of highest importance. A higher area under the curve (AUC) indicates better explanations. 
Similarly, the deletion curve tracks how much the output changes when the features are removed in reverse order. In this case, a low AUC indicates a good attribution. 
Although many different heuristics have been proposed to solve the attribution problem \cite{abhishek2022attributionbasedxaimethodscomputer}, none, to our knowledge, has approached it by optimizing the insertion and/or deletion curve directly. Assuming that faithfulness defined based on these curves is a desirable property of an attribution, directly optimizing the area under these curves would be a principled way of generating high-quality attributions.
In the following, we present an objective function  that can be used to directly optimize the AUCs  and an efficient approximation of its gradient that allows for real-time applications. The objective can be combined with other loss terms; however, we empirically show that the objective function alone is sufficient to provide faithful attributions that score high in terms of Average Sensitivity \cite{Robustness} and the Minimal Deletion Set metric  \cite{fong2017interpretable}, passing the MPRT (Model Parameter Randomization Test) \cite{Randomization}.

This study
\textbf{(i)} establishes a link  between top-$k$ feature selection with variable $k$ and the insertion and deletion curves;
    \textbf{(ii)} states the sample complexity of Monte Carlo approximations of the area under these curves;  \textbf{(iii)} proposes a differentiable operator for feature masking depending on the top-$k$ features selected via the ``Gumbel-top trick''~\cite{pmlr-v97-kool19a};
    \textbf{(iv)} adds the new operator to the Neural Explanation Mask (NEM) framework \cite{nemu} yielding \emph{Ranking-NEM} (\ranem), which produces sparse faithful explanations with low latency; and
    \textbf{(v)} experimentally evaluates \ranem{} applied to different deep learning architectures for image classification in comparison to state-of-the-art attribution methods, demonstrating that the new approach and performs, on average, superior in terms of faithfulness, inference speed,  and robustness.

\section{Related work}
The attribution of model input-output relations is a well-established and growing field. In the following, we provide a brief overview before explaining the NEM framework in more detail.
\subsection{Heuristics for solving the attribution problem}
\begin{figure}[ht]
    \centering\includegraphics[width=\linewidth]{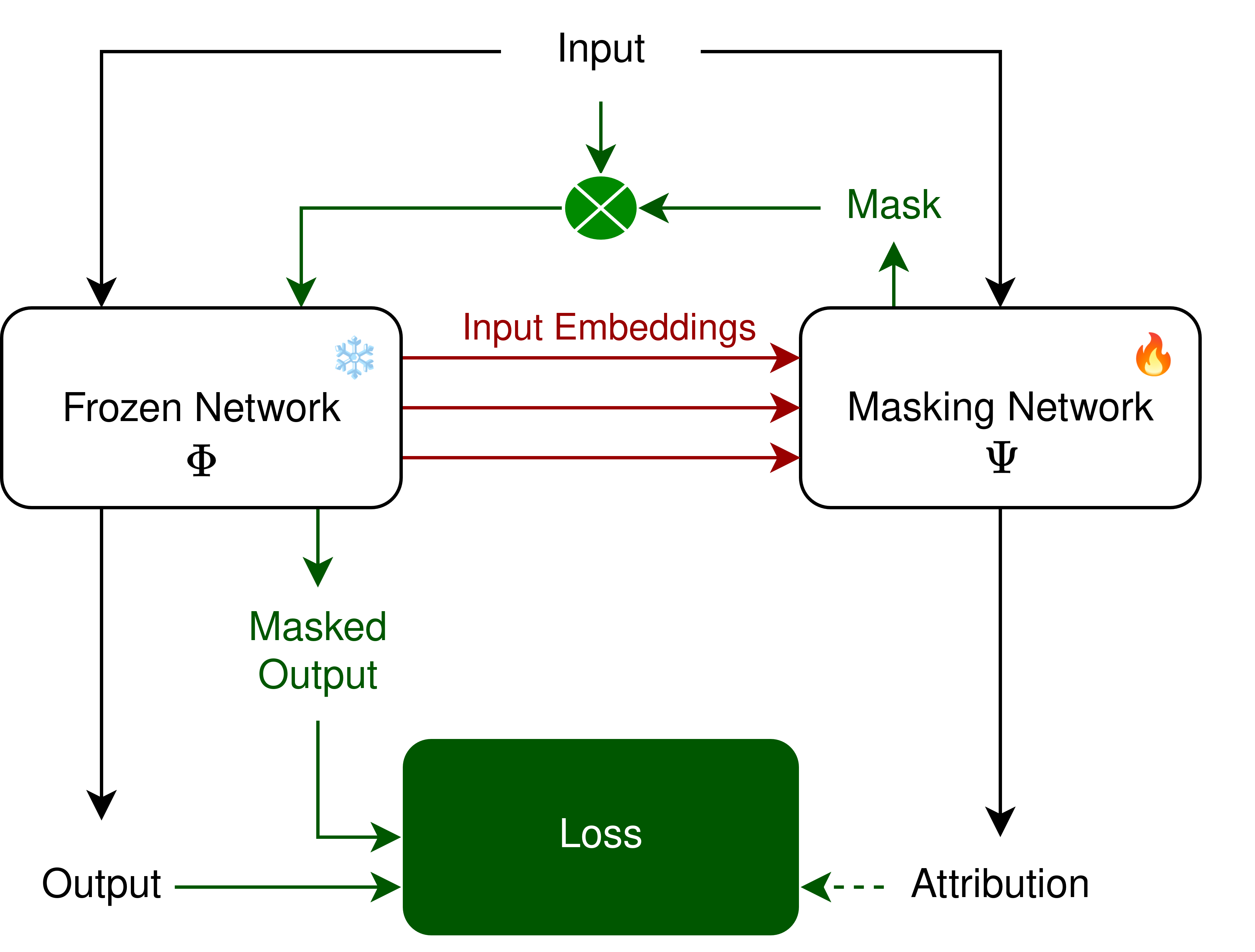}
    %\vspace{-1pt}
    \caption{The general Neural Explanation Method (NEM) framework's operation during inference (black and red arrows) and training (black, red, and green arrows). During inference, the input is processed by the frozen network, which generates its output, and the masking network, which produces an attribution (explanation). Depending on the NEM architecture, representations from the frozen network may assist the masking network in generating the attribution. During training, the masking network generates a mask from the attribution. In \cite{nemu,moller2024nemt}, this mapping is the identity,  and in this work it is a differentiable top-$k$ feature selector. The mask is applied to the input and passed through the frozen network to produce a masked output. Both masked and unmasked outputs are used in a loss function to optimize the masking network. Prior work has also incorporated the attribution into the loss.}
    \label{fig:framework}
\end{figure}
\label{sec:related_work}
% OCCLUSION VS gradient-based METHODS 
Many heuristics have been proposed to solve the attribution problem for different modalities such as images, texts, graphs, etc.~\cite{yang2023survey}. The approaches, especially in the context of computer vision, can be roughly divided into two families, occlusion- and gradient-based methods \cite{fong2017interpretable}. The latter leverages the gradient of the given model w.r.t.{} the input to determine the sensitivity of the output to the individual input features. These methods include Saliency \cite{Saliency}, GradSHAP \cite{lundberg2017unified}, GradCAM++ \cite{chattopadhay2018grad}, and Integrated Gradients \cite{sundararajan2017axiomatic}. The low latency of such methods usually comes with low faithfulness, especially for some specific neural architectures such as CNNs~\cite{pmlr-v235-muzellec24a}. The occlusion-based methods introduce perturbation/occlusion to the input and measure the changes in the output to determine the importance of various features. Examples include RISE~\cite{petsiuk2018rise}, Information Bottleneck Attribution (IBA, \cite{Schulz2020Restricting}), and Smooth Pixel Mask \cite{fong2019understanding}. Generally, occlusion-based methods exhibit good faithfulness at the cost of slower inference speed. A notable exception to this rule is the NEMt method \cite{moller2024nemt}, an instantiation of the Neural Explanation Mask (NEM) framework \cite{nemu}, which achieves occlusion-based masking at low latency by predicting a mask directly instead of running a costly optimization process for each input as done by other occlusion-based methods. 
%  
% ADDITIVE VS SET-OF-FEATURES AND L2X
XAI methods can also be categorized according to whether they generate additive feature explanations or set-of-feature explanations \cite{fong2019understanding, nemu}. Additive feature explanations provide attributions which rank individual features based on their importance. Examples of such methods include GradCAM \cite{chattopadhay2018grad}, RISE \cite{petsiuk2018rise}, and Integrated Gradients~\cite{sundararajan2017axiomatic}. In contrast, set-of-feature explanations identify the minimal subset of features that produce a similar output to the full feature set. Unlike additive explanations, these methods do not rank every feature, making them less suitable for evaluation with deletion and insertion curves. Examples include Smooth Pixel Mask \cite{fong2019understanding}, NEMt~\cite{moller2024nemt}, and Learn2Explain~\cite{L2X}. L2X uses top-$k$ feature selection, employing the Gumbel softmax operator iteratively to identify up to $k$ important features, where $k$ ranges from 4 to 10 in the original paper. To our knowledge, the framework has not been tested for larger values of $k$.

\subsection{Neural Explanation Masks}
\label{sub_sec:neural_explanation_masks}
% Introduce the NEM framework and the NEM_U architecture
Our approach extends the \emph{{N}eural {E}xplanation {M}ask} (NEM) framework \cite{nemu}. The idea of NEM is to augment a given differentiable model $\frozen$ with an explanation module $\masknet$ that generates a feature attribution $a = \masknet(x)$ for the model output $y = \frozen(x)$;  see \autoref{fig:framework} for an overview of the NEM framework.
%To achieve low latency, 
The explanation module $\masknet$ is trained on a corpus of (unlabelled) data. When deployed, providing an explanation corresponds to (simultaneously) evaluating $\masknet$, which leads to low latency.
%, allowing it to infer an explanation $m$ given new input. Note that 
The models $\frozen$ and $\masknet$ are {not} optimized in tandem,  which is why $\frozen$ is also called the \textit{frozen model}. This is in contrast to other occlusion-based methods, which usually need to perform a separate optimization process for each new input to generate an explanation. 
Different NEM architectures \cite{nemu,moller2024nemt} can be derived depending on the architecture of $\masknet$ and how $\masknet$ is integrated into the explained model $\frozen$. Previous work has focused on NEM-U (NEM using \unet{} \cite{ronneberger2015u}), where $\frozen$ and $\masknet$ are viewed as the encoder and decoder of a \unet{} architecture, respectively, meaning that $\masknet$ utilizes skip connections to extract intermediate input representations from $\frozen$. 
It is appealing that the NEM framework does not interfere with training and using $\frozen$, that is,
there is no need to adapt the model architecture and no risk of degrading the predictive performance.
All current NEM methods require us to choose hyperparameters that specify the trade-off between accuracy and complexity.
%, which requires the user to either experiment or have a good idea about the structure of the models latent space.
Furthermore, all previous NEM methods produce a set of important features. They can only determine whether a feature is important or not, but not how important it is compared to others. Thus, no previous NEM methods can be used to reliably determine the $k$ most important features for a given $k$.

\section{Methodology}
\label{sec:methodology}
This section presents our theoretical results and derives an XAI algorithm based on these results. First, we establish the relationship between the insertion and deletion curves and top-$k$ feature selection. This discussion lends itself to an objective function for learning feature attributions. To make the objective tractable, we use Monte Carlo approximation, and to make it differentiable, we employ
the ``Gumbel top-$k$ trick'' for sampling without replacement.
Then we show how the Neural Explanation Mask framework can be extended to our new \emph{Ranking-NEM} (\ranem) architecture using the derived objective function. 
\subsection{Approximating the insertion curve via top-\textit{k} feature selection}
\label{sub_sec:ranking_via_feature_selection}
The idea is to generate feature attributions by optimizing the AUC of the corresponding insertion (and/or deletion) curve to maximize the {faithfulness} of the attribution directly. 
Let $x \in \mathbb{R}^d$ be some $d$-dimensional input (e.g., a flattened image) and {$\frozen: \mathbb{R}^d \to \mathbb{R}^\outdim$}  some fixed model with $\outdim$-dimensional output  for which we want to provide explanations.
This $\outdim$-dimensional output could be the logits for $\outdim$ classes but also any other representation learned by $\frozen$, in particular some hidden layer embedding in a deep neural network.
The set of indices corresponding to the $d$ input features is denoted by $\I =\{1, \dots,d \}$.
For a given input $x$, a \emph{feature attribution} is a vector $\fa\in\mathbb{R}^d$, where $\fa_i < \fa_j$ indicates that the feature $i$ is considered more relevant than the feature $j$. Each feature attribution $\fa$ defines the corresponding \emph{feature ranking} $\sigma$. We describe the ranking by the permutation $\fr: \I \to \I $, which orders the features in descending order breaking ties according to some deterministic rule. That is, $\fr(i) < \fr(j)$ iff $\fa_i < \fa_j$.
\subsubsection{Insertion and deletion curves.}%\label{sec:curves}
Let $\mask: \mathbb{R}^d \times \pow(\I)  \to \mathbb{R}^d$, where $\pow$ denotes the power set, 
be some operator responsible for perturbing/masking features in $x$. The second argument specifies the features that should \emph{not} be perturbed. For example, the second argument could specify all pixels in an image $x$ that are not replaced by a constant when masking $x$.
%\footnote{In this study, we assume fixed length input vectors. However, the approach can also be adopted for variable length inputs. For example, the second argument could specify the set of elements (tokens) to be removed from an input sequence $x$.}
%
We define a function $\simm: \mathbb{R}^d \times \mathbb{R}^d  \to [0,1]$ 
%with $\simm(x,x)=1$,
that measures the quality of the perturbed input with respect to the original input and the model.
This could be the logit of the predicted class \cite{petsiuk2018rise} or, in an unsupervised setting, 
the similarity of latent space representations \cite{wickstrom2023relax,nemu}. 

The insertion curve is constructed by iteratively adding features to the empty set of features in the order of increasing importance according to some attribution and measuring the difference between the model output of the reduced feature set and the full feature set. We define the \emph{insertion curve} value at a given number of features $i$ by
\begin{equation}
\label{eq:add_curve}
% \dc(i,\sigma)    =  \simm(x, \mask(x, \{\sigmaInv(j)\,|\, 1\le j\le d - i    \}  )\enspace,  
\ic(i, x, \fr)    =  \simm(x, \mask(x, \{\frInv(j)\,|\, 1\le j\le  i    \}  )\enspace,
\end{equation}
That is, $\ic(i, x, \fr)$ measures the similarity of the outputs of $\frozen$ given $x$ and the results of adding the most important features according to feature attribution $a$.
The corresponding  AUC (area under the curve) is given by 
\begin{equation}
\label{eq:add_curve_auc}
 \operatorname{AUC}_\text{ins}(x,\fa)  =  {\frac{1}{d}}\sum_{i=1}^d \ic(i,x,\fr)\enspace.
\end{equation}
We define the complementary \emph{deletion curve} resulting from iteratively removing features in the order of decreasing relevance according to some attribution as
\begin{equation}
\label{eq:del_curve}
 \dc(i,x,\fr)  =  1-\simm(x, \mask(x, \{\frInv(j)\,|\, 1\le j\le d - i\}  )  %\enspace.  
%\ac(i,\sigma)  =  \simm(x, \mask(x, \{\sigma(j)\,|\, 1\le j\le i\}  )\enspace. 
\end{equation}
and $\operatorname{AUC}_\text{del}$ accordingly.
%Note that this is different from what is typically referred to as an \emph{insertion curve} in XAI. The AUC of the addition curve is given by 
%\begin{equation}
%\label{eq:add_curve_auc}
% \operatorname{AUC}_\text{add}(x,\fa)  =  \sum_{i=1}^d \ac(i,x,\fr) .
%\end{equation}
From  $\ic(i,x,\fr) = 1-\dc(d-i,x,\fr)$ it follows that $\operatorname{AUC}_\text{ins}(x,\fa) =1-\operatorname{AUC}_\text{del}(x,\fa)$ if the same $\mask$ is used for both curves.
However, typically different $\mask$ operators are used, denoted by $\maskdel$ and $\maskins$ in the following. In \cite{petsiuk2018rise}, $\maskdel(x,T)_i=0$ and
$\maskins(x, T)_i=\operatorname{blur}(x)_i$ for $i\in [d]\setminus T$,
where $\operatorname{blur}$ is a
Gaussian smoothing operator (we refer to \cite{petsiuk2018rise} for a discussion).
We follow this standard when computing \emph{faithfulness}. A third option is random replacement
$\maskrand(x, T)_i\sim\mathcal{N}(x)$ for $i\in [d]\setminus T$,
where $\mathcal{N}(x)$ is a Gaussian distribution with mean and variance matching the input statistics.

Using the insertion curve criterion, the optimal feature ranking is the permutation $\sigma^*$ 
\begin{equation}
\label{eq:add_curve_auc_argmax}
\sigma^* = \operatorname*{\arg\max}_{\sigma\in\mathcal{S}_d} 
\sum_{i=1}^d \ic(i,x,\sigma)
%\operatorname{AUC}_\text{ins}(x,\fa)
\enspace,
\end{equation}
where $\mathcal{S}_d$ is the set of all permutations of $[d]$, the positive integers up to $d$. To keep the notation simple, we assume that all $\arg\max$ operations return a unique element. However, the following considerations also hold if a set of equivalent solutions is returned.
They also hold when optimizing the area under the deletion curve or a combination $\operatorname{AUC}_\text{ins}(x,\fa) -\operatorname{AUC}_\text{del}(x,\fa)$ (with different choices of $\mask$), which follows from the equivalence of deletion and insertion curves derived above.

%\subsubsection{Monotonicity} 
%We assume monotonicity in the set of important features with respect to the model $\frozen$. 
The question arises under which conditions $\sigma^*$ defines the optimal subset of $n$ features in terms of $\simm$ for any $n\in\I$.
Let the set  of the $n$ most important features be defined as  
\begin{equation}
\label{eq:top_k_for_a_set}
T^*_n = \operatorname*{\arg\max}_{T\in\pow(\{1,\dots,d\}) \wedge |T|=n}   \simm(x, \mask(x, T ) )\enspace. 
\end{equation}
Then \emph{monotonicity in the set of important features} is defined as the property
%\begin{equation}
%\label{eq:monotonicity}
$T^*_n\subset T^*_{n+1}$. %\enspace.    
%\end{equation}
This monotonicity is an implication of a common more restrictive assumption in the attribution literature called \textit{additive feature attribution } \cite{lundberg2017unified}, which states that any given model output can be approximated by a linear combination of input features. 

Assuming monotonicity, there exists a permutation $\sigma_n$ with $\{\sigma^{-1}_n(i)\,|\,i=1,\dots,n\}=T_n^*$. 
Definition \eqref{eq:top_k_for_a_set} implies
$\simm(x, \mask(x, \{\sigma^{-1}(j)\,|\, 1\le j\le n   \}  ) \le \simm(x, \mask(x, T^*_n ))$
for all permutations $\sigma\in\mathcal{S}_d$.
Thus, $\sigma_n$ maximizes \eqref{eq:add_curve_auc}, that is, $\sigma^*=\sigma_n$.
Therefore, under the monotonicity assumption, we have
%\begin{equation}
%\label{eq:monotonicity_implies_top_k}
    $T_n^*= \{{\sigma^*}^{-1}(i)\,|\,i=1,\dots,n\}$,
%\end{equation}
that is, optimizing the AUC gives the optimal subset of features for any subset size.

\subsubsection{Monte Carlo approximation.}
The optimization problem defined by \eqref{eq:add_curve_auc_argmax} is typically very high-dimensional in practice, for example, in computer vision tasks where $d$ is the number of pixels in an image.
Thus, evaluating the insertion curve at all points during optimization becomes infeasible. Therefore, we consider the Monte Carlo approximation of $\operatorname{AUC}_\text{ins}$ given by 
\begin{equation}
\label{eq:add_curve_auc_approx}
 \widetilde{\operatorname{AUC}}_\text{ins}(x,\fa)  = \frac{{1}}{|\J|} \sum_{i\in \J} \ic(i,x,\fr)  \enspace,
\end{equation}
where $\J\subseteq\I$ is drawn uniformly at random. 
We have
$
\E[\widetilde{\operatorname{AUC}}_\text{ins}(x,\fa)] = \operatorname{AUC}_\text{ins}(x,\fa)$,
where the expectation is over draws of $\J$, and accordingly
$
\sigma^* = \arg\max_{\sigma\in\mathcal{S}_d} \E \left[ \widetilde{\operatorname{AUC}}_\text{ins}(x,\fa)\right]$.
%We provide a bound for the sample complexity of the Monte Carlo approximation in \autoref{appendix:sampling_bound}.

\subsubsection{Bounds on sample complexity.}%{Bounds on sample complexity}
\label{sampling_bound}
The sample complexity of the AUC approximation can be bounded as follows:
%using results from \cite{serfling:74}:
\newcommand{\calC}{\mathcal{C}}
\begin{theorem}\label{thm:approx}
Given an input $x$ and an attribution $\fa$ with corresponding permutation $\fr$,
for randomly drawn indices $\J$ we have for $t>0$
\begin{equation}
    P\big[|\widetilde{\operatorname{AUC}}_\text{ins}(x,\fa) - 
    {\operatorname{AUC}}_\text{ins}(x,\fa)|>t\big]\le 2e^{\frac{-2nt^2}{(1-f^*_n)(b-a)}}\enspace,
\end{equation}
    where $n=|\J|$; $a=\min\calC$ and $b=\max\calC$ for $\calC=\{c_{\text{ins}}(i,x,\fa)\,|\,i=1,\dots,d\}$; and $f^*_n=(n-1)/d$.
\end{theorem}
\begin{proof}
    For a fixed input $x$ and an attribution $\fa$ with corresponding permutation $\fr$, we have $\operatorname{AUC}_\text{ins}(x,\fa)=\mu=\frac{1}{d}\sum_{c\in\calC}c$. Setting $S_n = C_1 + \dots + C_n$, where the $C_1 , \dots , C_n$ are $n$ samples drawn uniformly from $\calC$ without replacement, we have $\widetilde{\operatorname{AUC}}_\text{ins}(x,\fa)=S_n/n$. Now we apply a two-sided version of Corollary 1.1 from \cite{serfling:74} using $P_n=P[S_n/n-\mu\ge t]$ and $P_n=P[\mu - S_n/n\ge t]$, which completes the proof.
\end{proof}
Note that $\simm$ does not need to be bounded for the proof.
Setting $f_n^*=0$ gives an expression that resembles what one would expect from a Hoeffding bound considering independent random variables. However, we cannot simply assume that the $c_{\text{ins}}(i,x,\fa)$ terms in the AUC calculation are independent. Anyway, the elements in $\calC$ are fixed given an input $x$ and an attribution $\fa$ with corresponding permutation $\fr$. Thus, we consider the problem of drawing without replacement from a set of numbers.

We can also take the variance into account, and, in the same way as above, from the results in \cite{serfling:74} we get:
\begin{corollary}
    With assumptions and definitions as in \autoref{thm:approx} and its proof, we have
\begin{equation}
    P\big[|\widetilde{\operatorname{AUC}}_\text{ins}(x,\fa) - 
    {\operatorname{AUC}}_\text{ins}(x,\fa)|>t\big]\le 2{\frac{(1-f_n)s^2}{nt^2}}\enspace,
\end{equation}
    with $\mu=\frac{1}{|\calC|}\sum_{c\in\calC}c$, $s^2=\frac{1}{|\calC|}\sum_{c\in\calC} (\mu-c)^2$ and $f^*_n=(n-1)/(d-1)$.
\end{corollary}
\subsection{A differentiable loss in terms of top-\textit{k} feature selection}
The first (top) $k$ elements of a feature attribution $\fa$ 
are given by $\text{top}_k(\fa) =  \{\frInv(i)\,|\,i=1,\dots,k \}$,
and we use the same notation for the $[k]\to[d]$ mapping 
%\begin{equation}
%\label{eq:top_k_feature_selection}
$\text{top}_k(\fa)(i) =  \fr(i)$.
%\enspace.
%\end{equation}
%that is, $\text{top}_k(\sigma*)=T_n^*$.
We can rewrite  \eqref{eq:add_curve_auc_approx} as
\begin{equation}
\label{eq:add_curve_auc_approx2}
 \widetilde{\operatorname{AUC}}_\text{ins}(x,\fa)  = \frac{{1}}{|\J|} \sum_{i\in \J} \ic(i, x, \text{top}_i(\fa))  \enspace.
\end{equation}
Instead of optimizing over the space of rankings,
we consider the real-valued optimization problem 
\begin{equation}\label{eq:astar}
    a^* = \operatorname*{\arg\max}_{a\in\mathbb{R}^d}   \widetilde{\operatorname{AUC}}_\text{ins}(x,\fa)\enspace.
\end{equation}
Again, we assume that $\operatorname*{\arg\max}$ returns a single solution to keep the notation simple. We have $\sigma^* = 
\operatorname*{\arg\max}_{\sigma\in\mathcal{S}_d} \widetilde{\operatorname{AUC}_\text{ins}}(x,\fa)
=  \sigma_{a*}$.
Plugging in \eqref{eq:add_curve}, the optimization problem \eqref{eq:astar} gives rise to the  loss function
\begin{equation}
\label{eq:losstop_nondiff}
 \loss(\fa\,|\,x)  = - \sum_{k\in \J} \simm(x,\mask(x,\text{top}_k(\fa))  %\enspace.
\end{equation}
to be minimized over $a\in\mathbb{R}^d$ for a training input $x$, where $\J$ is sampled anew in each evaluation of the loss.
Thus, instead of improving an attribution by optimizing the AUC of its insertion or deletion curve, we can instead optimize top-$k$ feature selection with uniformly sampled $k$ maximizing $\simm$.

Optimizing the AUC directly is a non-differentiable problem.
However, in the following we will derive  approximate gradients for the rephrased problem.
For gradient-based minimization of \eqref{eq:losstop_nondiff}, we need a differentiable approximation of the top-$k$ selection done by $\mask$.
Differentiable top-$k$ feature selection is commonly implemented using the Gumbel softmax operator \cite{gumbelsoftmax,L2X}. The Gumbel softmax operator represents a categorical distribution with a continuous and differentiable distribution \cite{maddison2017the}. However, the Gumbel softmax operator is insufficient for top-$k$ feature selection, as it samples only a single element from the categorical distribution. Previous work in the explainability literature \cite{L2X} addresses this limitation by applying the operator $k$ times to the network output, producing a set containing at most different $k$ elements. However, this approach becomes computationally infeasible for large $k$, since the number of operations scales linearly with $k$.
Alternatively, top-$k$ feature selection can be achieved by the ``Gumbel top-$k$ trick'' \cite{pmlr-v97-kool19a}. An ordered sample without replacement from a categorical distribution of size $k$ can be drawn by perturbing the (unnormalized) log-probabilities with values drawn from a standard Gumbel distribution and then selecting the top-$k$ largest elements. Let $a \in \mathbb{R}^d$ 
be a feature explanation, as, for example, produced by a neural network that provides an explanation for a given input $x$. We interpret $a$ as non-normalized log-probabilities that define the discrete probability distribution $p$ over $\I$ with 
\begin{equation}
    p(i) = \frac{\exp(a_i)}{\sum_{j \in \I} \exp(a_j)}
    \enspace.
\end{equation}
Let $\gum \in \mathbb{R}^d$  be a random vector drawn from a standard $d$-dimensional Gumbel distribution. Then $\text{top}_k(\fa + \gum)$
gives an ordered sample of $k$ elements drawn without replacement from the distribution $p$.
For a proof, we refer to \cite{pmlr-v97-kool19a}.
In contrast to iterative applications of the Gumbel softmax, this sampling can be efficiently computed for large $k$, making it suitable for high-dimensional data. 
Thus, \eqref{eq:losstop_nondiff} becomes
\begin{equation}
\label{eq:losstop}
 \ell(\fa\,|\,x, g, k)  = - \simm(x,\mask(x,\text{top}_k(\fa + g))  %\enspace.
\end{equation}
for a feature attribution $\fa$ given the input $x$, a number of features $k$, and a sample $g$ from the standard
$d$-dimensional Gumbel distribution.

In summary, by adopting differentiable top-$k$ feature selection via the Gumbel top-$k$ trick, we are able to create an efficient differentiable approximation to $\mask$, which allows us to optimize the  loss in \eqref{eq:losstop}.
\autoref{alg:top-k_operator-main} shows the implementation of the $\mask$ operator in pseudocode, details are provided in \autoref{appendix:omega_pseudo_code} in the supplementary material.
%, we discuss the concrete implementation of $\mask$ in more detail, including pseudocode. 

\begin{algorithm}[tb]
\caption{Omega operator in Python/PyTorch pseudocode, details are provided in \autoref{appendix:omega_pseudo_code} in the supplementary material.}
\label{alg:top-k_operator-main}
%\begin{center}
\begin{lstlisting}[style=python]
def omega_operator(A, x, k, R):
    # stabilize training
    A = A - A.logsumexp()
    # sample Gumbel noise
    U = torch.rand_like(A)
    G = -torch.log(-torch.log(U)) 
    # Gumbel perturbed sorting 
    A, indices = torch.sort(A + G, descending=True)
    # soft thresholding
    A = torch.sigmoid(A - A[k])
    # hard thresholding
    mask = A - A.detach()
    # only keep top k elements
    mask[:k] += 1
    # reorder mask to original index order
    M = torch.zeros_like(A)
    M[indices] = mask
    # generate perturbed input
    return x * M + R * (1 - M)  
\end{lstlisting}
%\end{center}
\end{algorithm}
\subsection{\ranem}\label{sec:ranem}
The loss function and its approximated gradient computation can be used as an additional term in various attribution algorithms. 
To show their effectiveness, we evaluate them within the NEM framework without additional loss terms.

The $\mask$ operator does not have any trainable parameters and simply selects an ordered subset from a given feature attribution in a differentiable way. 
We can combine it with  NEM-U in a straight-forward way. We feed the explanation produced by $\masknet$ into Algorithm \ref{alg:top-k_operator-main}. Then we can train $\masknet$ end-to-end by gradient-based optimization of \eqref{eq:losstop}. 
That is, for each input $x$ during training, we compute $\frac{\partial}{\partial\w} \ell(\masknet(x)\,|\,x, g, k)$ for randomly drawn $g$ and $k$, where $\w$ denotes the parameters (weights) of $\masknet$.
We refer to this new approach as \textit{Ranking-NEM} (\ranem), because it ranks all input features via the $\masknet$ output processed by Algorithm \ref{alg:top-k_operator}, which stand in contrast to previous NEM variants.  
\ranem{} applies binary masks during training (i.e., $\texttt{M}$ in \cref{alg:top-k_operator} is in $\{0,1\}^d$), while previous NEM methods work with continuous masks $m$, typically $m=a$ directly, and evaluate $\frozen$ on the element-wise product $x\odot m$. This is a crucial difference. While $x\odot m$ may be a valid input when processing images and other fixed-length continuous signals, this is not necessarily the case for other modalities, such as sequences of discrete symbols.
This makes the \ranem{} approach more flexible, for example, applicable in natural language processing.

%\paragraph{\ranem{} for classification.}
To train a \ranem{} for supervised classifiers, we 
follow \cite{petsiuk2018rise} and
set $\simm$ in \eqref{eq:losstop} to  the probability of the most likely class of the input image:
\begin{equation}
    \delta_\Phi(x,x')= [\Phi(x')]_c \text{ with } 
    c=\arg\max_i \frozen(x)_i
\end{equation}
During training, for each input $x$, $\mask$ is chosen uniformly at random to be $\maskins$, $\maskdel$, or $\maskrand$ (see \autoref{sub_sec:ranking_via_feature_selection}).

\section{Experiments \& results}
\label{sec:experiments_and_results}
To evaluate the \ranem{} method, we performed a series of experiments.

\subsection{Evaluated architectures and dataset}
We applied a diverse set of explanation methods to four deep learning models. Specifically, we studied three convolution-based architectures, ResNet50 \cite{he2016deep}, ConvNeXt (small) \cite{liu2022convnet}, and VGG16 \cite{simonyan2014very}, and one transformer architecture, ViT \cite{dosovitskiy2020image}. All models were sourced from the Timm library \cite{rw2019timm} and pretrained on the ImageNet \cite{imagenet} training split.
We focused on attribution methods for supervised image classifiers \cite{abhishek2022attributionbasedxaimethodscomputer}. We ran evaluations using 1,000 images sampled from the ImageNet \cite{imagenet} validation data. For training NEMs, we sampled 10,000 images from the same split, discarding images that overlapped with the evaluation dataset.

\subsection{Baseline methods}
To compare our proposed methodology, we evaluated nine explanation methods that included gradient-based and occlusion-based approaches. From gradient-based methods, we considered \saliency{} \cite{Saliency}, \gradshap{} \cite{lundberg2017unified}, \gradcam{} \cite{chattopadhay2018grad}, and \intgrad{}\cite{sundararajan2017axiomatic}. From occlusion-based methods, we examined \rise{} \cite{petsiuk2018rise}, \spm{} \cite{fong2019understanding}, Information Bottleneck Attribution (\iba) \cite{Schulz2020Restricting}, and the recent \nemt~\cite{moller2024nemt}.
For \iba{}, we picked the \textit{per-sample variant}, optimizing directly on the input image as recommended in the literature to achieve the highest faithfulness. We could not find any work on \iba{} for transformer-based image architectures and therefore did not generate results for explaining ViT with this method. For the \spm, we enforced an area constraint of 10\%.

\subsection{FORgrad}
In \cite{pmlr-v235-muzellec24a}, it has been shown  that the poor performance of  gradient-based attribution methods applied to CNNs is likely due to the pooling operations leveraged in CNNs. The authors  introduce FORgrad, which is a method for improving the faithfulness of gradient based methods by filtering out high frequency noise in the gradient. We add the FORgrad method to the gradient-based methods in our study and add the resulting algorithms to our comparison.

\subsection{NEM training and post processing}We studied two different NEM architectures, \ranem{}, and \nemt. All NEM masking networks $\masknet$ were standard U-Net decoder blocks, consistent with previous work \cite{nemu,moller2024nemt}, to ensure a fair comparison. The trainable neural architectures used for the two methods were identical, which means that the number of trainable parameters was the same. All masking networks were trained for 10 epochs on the training data using the Prodigy optimizer \cite{mishchenko2024prodigyexpeditiouslyadaptiveparameterfree} alongside a cosine annealing learn-rate scheduler. %To enhance training stability for 
In \ranem, we sampled six values of $k$ with replacement for each input $x$ in a batch and optimized for each sampled value, that is, we considered six loss terms \ref{eq:losstop} for each $x$.  We set $\lambda = \frac{1}{50}$ in the \nemt{} loss function.  A reproducibility checklist can be found in  \autoref{appendix:repro_check_list} and a  hyperparameter study in \autoref{appendix:training_nem_models} of the  supplementary material.

\subsection{Performance metrics}
We compared the different XAI methods using metrics for faithfulness, complexity, and inference time.
Faithfulness was assessed using insertion and deletion curves, as is commonly done in the XAI literature \cite{petsiuk2018rise,pmlr-v235-muzellec24a,fong2019understanding}. We use the variant proposed by \cite{pmlr-v235-muzellec24a}. Specifically, for each attribution-image-model combination, we compute the insertion and deletion curves and calculate the respective AUCs using 
$\maskdel(x,T)_i=0$ and
$\maskins(x, T)_i=\operatorname{blur}(x)_i$ for $i\in [d] \setminus T $. The final faithfulness metric is obtained by subtracting the deletion curve AUC from the insertion curve AUC. 

Previous work  has shown that  measuring faithfulness using deletion  and insertion curves can be sensitive to the choice of hyperparameters \cite{wickstrøm2024flexibilitymanipulationslipperyslope}. This is why we use two different choices of perturbation during our evaluation. To further explore how this might affect our results, we measure how different perturbation choices during faithfulness evaluation lead to different scores. 

There exist alternative measures of faithfulness, especially the Minimal Deletion Set metric  \cite{fong2017interpretable}, which measures the smallest perturbation that changes the predicted class of the classifier. Thus, we also studied this metric for the best-performing methods.

To further compare the attribution methods considered in this study,  we measure their robustness to (unimportant) noise using the Average Sensitivity metric \cite{Robustness}. A common sanity check of attribution methods is to measure how much the attribution for the same input changes when the model is 
%progressively 
distorted by injecting random noise into the layers. We performed this check using a variant of the MPRT (Model Parameter Randomization Test \cite{Randomization}) from the Quantus library \cite{hedstrom2023quantus},
where we measure the average spearman rank correlation between the attributions produced by the original model and a model with randomized weights. The full evaluation dataset (1000 images)  was used to measure Faithfulness, %Sparsity and Complexity,
whereas a subset of 100 images was used to measure Randomization and Robustness due to the metrics being computationally demanding. Finally, we measured the average time in seconds to generate attributions for the full evaluation dataset. Generally, low-latency methods are preferable, as real-time inference may be a requirement in practical applications.  All experiments were run on an RTX4090 NVIDIA GPU.

\subsection{Results}

%\label{sub_sec:results}
The results of our experiments comparing nine XAI methods are summarized in \autoref{tab:main_res},  the details are provided in \autoref{appendix:granular_results} of the supplementary material. All differences in faithfulness between \ranem{} and the other methods in \autoref{tab:main_res} are highly statistically significant (two-sided
paired Wilcoxon test, $p < 0.001$). In our empirical evaluation,  \ranem{} achieved the highest average faithfulness. It gave the highest faithfulness for the ResNet50, ConvNeXt, and ViT models, while ranking fifth on the VGG16 model. 

\begin{table}[ht]
\caption{Aggregated results from running different XAI methods on four different models using 1000 samples of the validation split of the ImageNet dataset. The explained models  are a ConvNeXt, a ResNet50, a VGG16, and a ViT. 
We  compared Faithfulness of \ranem{} with the other methods, and the differences in the table below are statistically highly significant (two-sided paired Wilcoxon rank-sum test, $p<0.001$).
Results for each individual model are given in \autoref{appendix:granular_results} in the supplementary material. $^\dagger$IBA was only evaluated for the three CNN architectures. $^\ddagger$\gradcam{} was only evaluated for Robustness on the three CNNs and \nemt{} was only evaluated on ViT and ResNet50. Please see \autoref{appendix:granular_results} for further information.}\vspace{5pt}
\label{tab:main_res}

\centering
\newcommand{\ze}{\ensuremath{\scriptstyle\boldsymbol{\approx}\boldsymbol{0}}}

\newcommand{\zesa}{\ensuremath{\scriptscriptstyle\mathrel{\substack{\scriptscriptstyle{\downarrow}\\0\\\scriptscriptstyle{\uparrow}}}}}

\newcommand{\zes}{\raisebox{0.5ex}{\scalebox{0.5}{\ensuremath{\scriptscriptstyle{\mathrel{\substack{\downarrow\\0\\\uparrow}}}}}}}

\begin{tabular}{@{}lrrrr@{}}
\toprule
\textbf{Method} & \textbf{Faith. $\uparrow$} &  \textbf{Robust. $\downarrow$} & \textbf{Rand. \zes} & \textbf{Time $\downarrow$} \\ \midrule
\rise            & 0.422                                         & 0.362                     & -0.017                              & 12.835                     \\
\gradcam{}${}^\ddagger$         & 0.390                          & 0.599                     & -0.082                              & 0.013                      \\
\intgrad         & 0.374                                         & 1.298                     & \phantom{-}0.016                               & 0.092                      \\
\spm & 0.386                                          & 0.730                      & \phantom{-}0.138                               & 4.672                      \\
\gradshap        & 0.343                                         & 1.350                      & \phantom{-}0.015                               & 0.014                      \\
\iba{}${}^\dagger$   & 0.455                                      & 0.189                    & -0.222                              & 0.240                       \\
\saliency        & 0.294                                        & 1.332                     & \phantom{-}0.075                               & 0.008                      \\
\nemt{}${}^\ddagger$            & 0.423                                          & 0.435                     & \phantom{-}0.186                              & 0.005                      \\
\ranem           & \textbf{0.494}                                          & \textbf{0.152}                     & \phantom{-}0.051                              & \textbf{0.004}                      \\ \bottomrule
\end{tabular}
\end{table}

With respect to Robustness, it ranks first overall indicating noise-resistant explanations. With a low Randomization score (0.051), \ranem{} proves to be sensitive to the underlying model as desired and thus passes the sanity check. 

\ranem{} and \nemt{} were the fastest methods, with our implementation of \ranem{} being the most efficient. The training 
times for different explained models and both NEM methods 
are reported in \autoref{tab:nem_training_time} in the supplementary material.

\begin{table}[ht!]
\centering
\caption{Average ratio of input that has to be perturbed before the prediction of the base model  differs from the the original.
Lower is better. Perturbation method is zero imputation.}
\label{tab:mr_res_minimal_deletion_set}
%\resizebox{0.75\columnwidth}{!}
{%
\begin{tabular}{@{}lrr@{}}
\toprule
                     & Resnet50 & ConvNeXt \\ \midrule
RISE                 & 0.242    & 0.438    \\
Integrated Gradients & 0.313    & 0.532    \\
Ra-NEM               & \textbf{0.168}    & \textbf{0.212}    \\ \bottomrule
\end{tabular}%
}
\end{table}

In terms of the Minimal Deletion Set metric  \cite{fong2017interpretable}, \ranem{} also achieves the best score compared to the other high-performing methods, see \autoref{tab:mr_res_minimal_deletion_set}. 

Visual examples of the various attribution methods are shown in \autoref{fig:resnet500} with additional images provided in  \autoref{appendix:model_images} in the supplementary material. We provide visual examples of \ranem{} performing relatively poorly in \autoref{appendix:poor_performance} of the supplementary material. 

For the CNN-based explained (frozen) models, the occlusion based XAI methods (e.g., \rise{}) compared favorably to the gradient based approaches (e.g., \intgrad{}) in our experiments.
Adding the FORgrad method indeed improves the faithfulness results of gradient-based methods, as can be seen in \autoref{tab:forgrad_results}, but not to the extent that they reach the quality of \ranem{}.

\begin{table}[H]
\centering
\caption{Comparing faithfulness on ResNet50 
for \ranem{} and gradient-based methods with and without FORgrad. We can see that FORgrad improves faithfullness at the cost of complexity.  \ranem{} is still outperforming the gradient-based methods even after they have been ``repaired'' by FORgrad.}
\vspace{5pt}
\label{tab:forgrad_results}
\begin{tabular}{@{}lr@{}}
\toprule
\textbf{Method}   & \textbf{Faithfulness $\uparrow$}   %\textbf{Time $\downarrow$} 
\\ \midrule
\gradshap          & 0.343                              %        & 0.014                      
\\
\gradshap{} (FORgrad) & 0.354                           %           & 0.012                      
\\
\intgrad           &  0.374                             %          & 0.092                      
\\
\intgrad{} (FORgrad)  & 0.407                           %        & 0.083                      
\\
\saliency          & 0.294                              %        & 0.008                      
\\
\saliency{} (FORgrad) & 0.308                           %         & 0.006                      
\\
\ranem             & \textbf{0.607}                     %                  & \textbf{0.004 }                     
\\ 
\bottomrule
\end{tabular}

\end{table}

\begin{figure*}[ht]
    \centering
    \includegraphics[width=\linewidth]{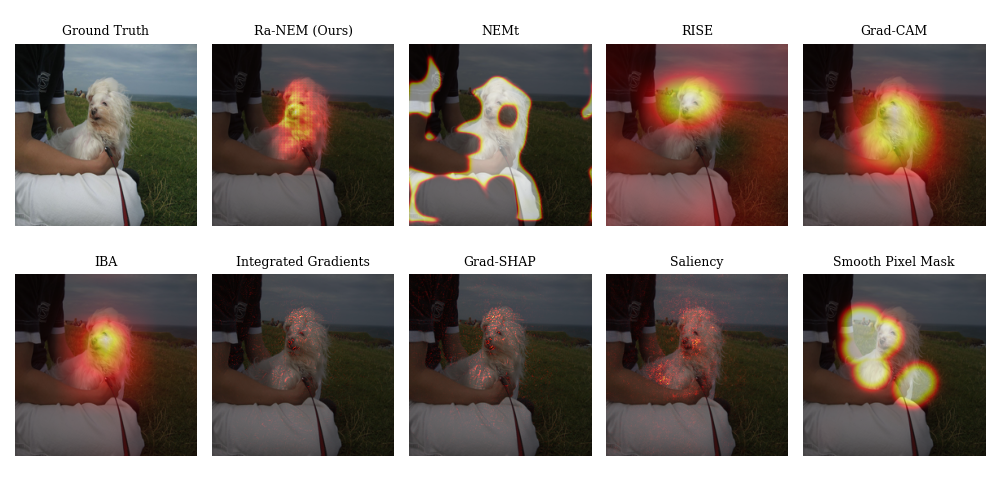}
    \vspace{-25pt}
    \caption{Model explanations generated by  different XAI methodologies (see appendix for more examples). The explained model is a ResNet50 architecture trained on the training split of the ImageNet dataset. The example image  is taken from the ImageNet validation dataset. Compared to other methods, such as \rise{} and \gradcam{}, \ranem{} generates a ranking more closely aligned with the central object (dog), while the others exhibit a circular bias, likely due to smoothing. This may explain the higher faithfulness of \ranem{} (see \autoref{sec:discussion}).}
    \label{fig:resnet500}
\end{figure*}

 The results of our hyperparameter study regarding the role of different perturbation methods are shown in \autoref{tab:mr_faithfulness_sensitivity}. The scores are indeed sensitive to the perturbation method, but the ranking appears to be stable across the models.  
\begin{table}[!ht]
\centering
\caption{Effect of the choice of perturbation  when calculating the deletion/insertion AUC, the shown values are $\text{AUC}_\text{ins} - \text{AUC}_\text{del}$ (ConvNeXt).}
\label{tab:mr_faithfulness_sensitivity}
%\resizebox{0.7\linewidth}{!}
{%
\begin{tabular}{@{}lrrr@{}}
\toprule
Method & Noise & Zero & Blur \\
\midrule

\rise    & 0.266 & 0.267 & 0.198 \\
\intgrad & 0.138 & 0.191 & 0.160 \\
\ranem{}   & \textbf{0.429} & \textbf{0.466} & \textbf{0.261} \\
\bottomrule
\end{tabular}%
}
\end{table}

The effect of the number of samples on faithfulness and training times for the \ranem{} method when explaining a ResNet50 model are shown in \autoref{tab:num_samples} in the supplementary material. Increasing $k$ can improve faithfulness but also increase training time. 
The stability of \ranem{} training is shown in \autoref{tab:nem_training_stability} in the supplementary material.

\section{Discussion}
\label{sec:discussion}

We derived a method for generating attributions by optimizing a Monte Carlo estimate of the AUC of the insertion/deletion curve via top-$k$ feature selection with uniformly sampled $k$ during training. 

To show the practical relevance of these results, we evaluated them within the NEM framework. 
Our experiments indicate that the resulting \ranem{} algorithm is a robust choice to provide faithful attributions as measured by deletion and insertion curve AUC. It has a low latency that allows for real-time applications, is robust to small input perturbations and sensitive to the underlying model. 

\subsection{Faithfulness} In our experiments, \ranem{} achieved on average the highest faithfulness for all image classifiers. There was some variance in performance between models, but \ranem{} had the highest faithfulness on three of four models. \ranem{} gave the best result on the ConvNeXt model, where it improved 0.095 over \nemt{} ranking second.  The worst relative performance of \ranem{} was on VGG16, where it was 0.139 behind the first place taken by \rise. 
The strong average performance of \ranem{} can be attributed to the inconsistent performance of other XAI methods when applied to different models, such as the poor performance of \rise{} on ViT and ConvNeXt. In contrast, our method gave faithful explanations for all models, even if it was outperformed for some individual models. To understand why our method performs well, it is therefore interesting to explore why the other methods vary in performance.\autoref{fig:resnet500} exemplifies that the \ranem{} attributions were more centered on the object (dog), while other occlusion-based methods have more circular smeared explanations. This is most likely due to smoothing, whether explicitly in the case of \iba{} \cite{Schulz2020Restricting} and \spm{} \cite{fong2019understanding} or implicitly by RISE \cite{petsiuk2018rise}. This bias toward spatially cohesive masks can be good for some CNN-based methods such as VGG16 and ResNet50, but it might hamper the ability to explain methods using patching such as ViT and ConvNeXt.  Previous work \cite{pmlr-v235-muzellec24a} has shown that gradient-based methods tend to exhibit reduced faithfulness when applied to CNN architectures due to the pooling operations.  Thus, it could be argued that the good overall performance of \ranem{} is partly due to not relying on spatial biases or gradients of the explained model.
We hypothesize that the most important reason for \ranem{} not achieving the highest faithfulness in all settings is that, unlike most other methods we compared against, \ranem{} does not optimize directly on the output when generating attributions after training. Instead, it relies solely on prediction, and therefore can sometimes focus on the wrong areas (see \autoref{appendix:poor_performance} in the supplementary material).

\subsection{Randomization and robustness} Our empirical results indicate that \ranem{} is the most robust of the studied methods, where robustness against input noise is measured by Average Sensitivity.
It even outperforms \iba{} and \gradcam{}, which use smoothing to eliminate some of the noise inherent to gradient-based methods. Furthermore, \ranem{} passed the randomization sanity check, since the MPRT score is close to zero (0.051). This means that the attribution depends on the underlying model mechanics instead of learning some shortcut such as segmenting the object in the image. Since \ranem{} only has access to the latent representations of the input and not the input itself, this is not surprising.

\subsection{Time}  \ranem{} was the fastest method for all models, closely followed by \nemt{}. This is in agreement with previous work on NEM methods \cite{nemu, moller2024nemt}, which outpaced other XAI methods when deployed. The NEM approaches only need a forward pass through the explained network $\frozen$ and the masking network $\masknet$, whereas the third fastest method \saliency{} needs to compute the gradient of the explained model. Since the $\masknet$  network is generally smaller than $\frozen$, a forward pass through both  $\masknet$ and $\frozen$ is  faster than a forward and backward pass through $\frozen$. %(even when performed sequentially).

\subsection{\ranem{} and previous NEM methods} Compared to previous NEM algorithms \cite{nemu,moller2024nemt}, our method offers advantages beyond the reported improved performance results. Specifically, the \ranem{} feature ranking and masking eliminate the need to specify a trade-off between complexity and accuracy. The binary masking facilitates adaptation to other input modalities, such as text or graphs, where partial occlusion is not well defined (see \autoref{sec:ranem}). Additionally, the absence of tuning parameters in the loss function simplifies adoption, as there is no need to adjust hyperparameters when $\frozen$ is re-trained or the framework is applied in a new setting.

\subsection{Using the  insertion/deletion curve as a target}
%A central assumption to this work is that the 
The insertion/deletion curves are widely considered as important measures of whether an attribution captures the underlying model mechanics in the XAI literature. Therefore, we studied the AUCs of these curves and showed how to efficiently optimize them.  
It has been pointed out that feature removal metrics are sensitive to the choice of perturbation, how many perturbation steps are taken during feature removal, and whether the order of feature removal is ascending or descending in the attribution values \cite{tomsett2020sanity}. We mitigate these issues by considering both the insertion and deletion curves, which rely on different perturbation schemes. Furthermore, \ranem{} uses three different perturbation schemes during training. %Furthermore, by measuring both the insertion and the deletion curve we have examples of deleting the largest or the smallest feature first. 
In addition, \ranem{} is optimized using a random $k$ during training, which means that it is not optimized with an a priori fixed number of steps. Our results show that these measures indeed lead to a high robustness of \ranem{}, indicating  that the results generalize. Finally, we would like to stress that optimizing the deletion/ insertion curve  can be seen as a generalization of the objective of finding a good trade-off between complexity and accuracy, a common way of creating explanations \cite{nemu, fong2019understanding,pmlr-v235-muzellec24a}, as the optimal insertion/deletion curve is one that has optimal accuracy for a given complexity level as defined by $k$.   

\subsection{Monotonicity}
Given monotonicity, we prove that the stochastic top-$k$ feature selection approximates faithfulness measured via
the deletion curve. Monotonicity is arguably unrealistic for many real world models. However, 
this assumption is not required when applying \ranem{} in practice. 
The proposed method is reasonable even if monotonicity does not hold: Our optimization process learns attributions such that, for arbitrary
 subset sizes, the selected features approximate the original input in terms of model prediction.

Monotonicity is implied by the more restrictive assumption of \textit{additive feature attributions} used in many attribution methods such as \gradshap{} \cite{lundberg2017unified} and LIME \cite{ribeiro2016should}. Evaluating and visualizing any XAI method using a single gradual map of relative feature importance (such as in \autoref{fig:resnet500}) implicitly assumes {monotonicity}. 
 
 It can be argued that all approaches providing a single attribution map for an input implicitly assume an ordering of the input features, which in turn implies monotonicity, otherwise, different maps for different $k$ must be reported.

\subsection{Fairness of comparison}
\spm{} and \nemt{} belong to a different category of explanation methods than the other approaches considered in this study, which may introduce bias in direct comparisons \cite{nemu,fong2019understanding}. Specifically, these methods are classified as set-of-feature explanations, whereas the others fall under additive feature explanations~\cite{nemu}. Set-of-feature explanations identify a minimal set of features sufficient to explain an outcome, but do not rank individual features. In contrast, additive feature explanations assign importance scores to all features, allowing for feature ranking. Comparing these approaches is inherently challenging, as they address different problems \cite{nemu}. We followed previous works \cite{pmlr-v235-muzellec24a, petsiuk2018rise, fong2019understanding} in our evaluation strategy to allow for a direct comparison, but future work should explore whether these results generalize to other domains. 

The \ranem{} method induces an ordinal ranking of the input features, meaning that the absolute values of the elements do not reflect the actual reduction in the classifier, but simply the correct ordering. Some faithfulness criteria such as \cite{wrongfaith_1} and \cite{wrongfaith_2} basically assume a linear relationship between classifier change and attribution. We have focused on faithfulness metrics that test ordinal relationships.

\section{Conclusions}
%\subsection{Summary}
%and future work}
%\label{sec:conclusion}
The areas under the insertion/deletion curves are often considered as important measures to evaluate the faithfulness of attribution-based XAI methods. This study contributes to a better conceptual understanding of this approach. We derived the sample complexity of a Monte Carlo approximation of the AUCs and an efficient differentiable objective function that allows one to
optimize the AUC directly. 
These general findings can be used to improve existing XAI methods. We added the new objective function to the NEM approach. This leads to the
\ranem{} algorithm, which performs top-$k$ feature selection and optimizes faithfulness to produce more accurate attributions.  In our experiments,  \ranem{} achieved state-of-the-art performance in terms of both faithfulness and robustness as measured by different performance metrics. Its low latency allows for real-time processing. 
Applications with real-time constraints or that require generating numerous explanations, such as online video analysis, benefit not only from the high accuracy and low latency of  \ranem{}, but also from its resource efficiency.
\ranem{} extends the NEM framework and makes it applicable to a wider range of input modalities, which is an interesting direction for future work.

\bibliographystyle{IEEEtran}
\bibliography{main}

\newpage 
\clearpage

\appendices
\section{Detailed  results}
\label{appendix:granular_results}
Tables \ref{tab:res_resnet} to \ref{tab:res_vit}  report the main results for each individual model explained.

% \usepackage{booktabs}
\newcommand{\zes}{\raisebox{0.5ex}{\scalebox{0.5}{\ensuremath{\scriptscriptstyle{\mathrel{\substack{\downarrow\\0\\\uparrow}}}}}}}

\begin{table}[H]
\label{tab:resnet50_res}
\centering
\caption{Metrics from running nine different XAI methods on a ResNet50 architecture using 1000 samples of the validation split of the ImageNet dataset. We  compared Faithfulness, Complexity and Sparsity results of \ranem{} with the other methods, and the differences in the table below, excluding RISE when measuring faithfulness, are statistically highly significant (two-sided paired Wilcoxon rank-sum test, $p<0.001$).\label{tab:res_resnet}}
%\vspace{5pt}
%\resizebox{\linewidth}{!}
{%
\begin{tabular}{@{}lrrrr@{}}
\toprule
\textbf{Method} & \textbf{Faith. $\uparrow$} &
%\textbf{Complex. $\downarrow$} & \textbf{Sparse. $\uparrow$} &
\textbf{Robust. $\downarrow$} & \textbf{Rand. \zes} & \textbf{Time $\downarrow$} \\ \midrule
\rise{}         & 0.568                                           & 0.263                         & -0.034                      & 9.235                      \\
\gradcam{}      & 0.484                                           & 0.284                         & 0.240                        & 0.009                      \\
\intgrad{}      & 0.397                                              & 1.360                          & 0.017                       & 0.046                      \\
\spm{}          & 0.436                                              & 0.729                         & 0.032                       & 3.349                      \\
    \gradshap{}     & 0.389                                           & 1.088                         & 0.017                       & 0.008                      \\
\iba{}          & 0.531                                              & \textbf{0.184}                         & -0.068                      & 0.117                      \\
\saliency{}     & 0.396                                             & 0.893                         & 0.096                       & 0.007                      \\
\nemt{}         & 0.505                                             & 0.169                         & 0.302                      & 0.006                      \\
\ranem{}        & \textbf{0.607}                                             & 0.204                          & 0.091                      & \textbf{0.003}                      \\ \bottomrule
\end{tabular}
}
\end{table}
% \usepackage{booktabs}

\begin{table}[H]
\centering
\caption{Metrics from running nine different XAI methods on a ConvNeXt architecture using 1000 samples of the validation split of the ImageNet dataset. We  compared Faithfulness, Complexity and Sparsity results of \ranem{} with the other methods, and the differences in the table below are statistically highly significant (two-sided paired Wilcoxon rank-sum test, $p<0.001$).}
%\vspace{5pt}
\label{tab:convnext_res}
%\resizebox{\linewidth}{!}
{%
\begin{tabular}{@{}lrrrr@{}}
\toprule
\textbf{Method} & \textbf{Faith. $\uparrow$} &  
\textbf{Robust. $\downarrow$} & \textbf{Rand. \zes} & \textbf{Time $\downarrow$} \\ \midrule
\rise{}         & 0.253                                           & 0.395                         & -0.015                      & 14.203                     \\
\gradcam{}      & 0.321                                             & \textbf{0.107}                         & -0.002                      & 0.015                      \\
\intgrad{}      & 0.332                                              & 1.756                         & 0.012                       & 0.081                      \\
\spm{}          & 0.251                                             & 0.744                         & 0.275                       & 5.026                      \\
\gradshap{}     & 0.311                                             & 2.073                         & 0.010                        & 0.012                      \\
\iba{}          & 0.311                                             & 0.079                         & -0.296                      & 0.486                      \\
\saliency{}     & 0.158                                             & 2.667                         & 0.008                       & 0.009                      \\
\nemt{}         & 0.44                                             & 0.532                         & 0.069                      & 0.006                      \\
\ranem{}        & \textbf{0.468}                                               & 0.168                         & -0.096                      & \textbf{0.005}                      \\ \bottomrule
\end{tabular}
}
\end{table}
% \usepackage{booktabs}
% \usepackage{booktabs}

\begin{table}[H]
\centering
\caption{Metrics from running nine different XAI methods on a VGG16 architecture using 1000 samples of the validation split of the ImageNet dataset. We  compared Faithfulness, Complexity and Sparsity results of \ranem{} with the other methods, and the differences in the table below, excluding NEMt when measuring faithfulness, are statistically highly significant (two-sided paired Wilcoxon rank-sum test, $p<0.001$).}
%\vspace{5pt}
\label{tab:vgg_res}
%\resizebox{\linewidth}{!}
{%
\begin{tabular}{@{}lrrrrrr@{}}
\toprule
\textbf{Method} & \textbf{Faith. $\uparrow$} & 
%\textbf{Complex. $\downarrow$} & \textbf{Sparse. $\uparrow$} &
\textbf{Robust. $\downarrow$} & \textbf{Rand. \zes} & \textbf{Time $\downarrow$} \\ \midrule
\rise{}         & \textbf{0.532}                                          & 0.354                         & -0.017                      & 16.599                     \\
\gradcam{}      & 0.502                                            & 0.378                         & -0.485                      & 0.013                      \\
\intgrad{}      & 0.313                                             & 0.975                         & 0.017                       & 0.115                      \\
\spm{}          & 0.474                                             & 0.670                          & 0.033                       & 3.873                      \\
\gradshap{}     & 0.31                                            & 1.019                         & 0.017                       & 0.018                      \\
\iba{}          & 0.522                                            & 0.305                         & -0.301                      & 0.358                      \\
\saliency{}     & 0.296                                             & 0.923                         & 0.053                       & 0.006                      \\
\nemt{}         & 0.304                                            & 0.603                         & {\NAN}                      & 0.004                      \\
\ranem{}        & 0.395                                             & \textbf{0.179}                         & 0.182                      & \textbf{0.002 }                     \\ \bottomrule
\end{tabular}
}
\end{table}
% \usepackage{booktabs}

\begin{table}[H]
\centering
\caption{Metrics from running nine different XAI methods on a ViT architecture using 1000 samples of the validation split of the ImageNet dataset. We  compared Faithfulness, Complexity and Sparsity results of \ranem{} with the other methods, and the differences in the table below, excluding Integrated Gradients when measuring faithfulness, are statistically highly significant (two-sided paired Wilcoxon rank-sum test, $p<0.001$).\label{tab:res_vit}}
%\vspace{5pt}
\label{tab:vit_res}
%\resizebox{\linewidth}{!}
{%
\begin{tabular}{@{}lrrrr@{}}
\toprule
\textbf{Method} & \textbf{Faith. $\uparrow$} 
%& \textbf{Complex. $\downarrow$} & \textbf{Sparse. $\uparrow$} 
& \textbf{Robust. $\downarrow$} & \textbf{Rand. \zes} & \textbf{Time $\downarrow$} \\ \midrule
\rise{}         & 0.336                                             & 0.438                         & -0.003                      & 11.303                     \\
\gradcam{}      & 0.254                                             & 1.626                         & {\NAN}                      & 0.012                      \\
\intgrad{}      & 0.456                                             & 1.104                         & 0.017                       & 0.124                      \\
\spm{}          & 0.384                                             & 0.777                         & 0.213                       & 6.441                      \\
\gradshap{}     & 0.362                                             & 1.221                         & 0.016                       & 0.017                      \\
\iba{}          & {\NAN}                                           & {\NAN}                        & {\NAN}                      & {\NAN}                     \\
\saliency{}     & 0.325                                             & 0.846                         & 0.143                       & 0.009                      \\
\nemt{}         & 0.443                                             & 0.246                         & 0.042                       & 0.006                      \\
\ranem{}        & \textbf{0.506}                                             & \textbf{0.058}                         & 0.029                       & \textbf{0.005}                      \\ \bottomrule
\end{tabular}
}
\end{table}

\section{Training NEM models}
\label{appendix:training_nem_models}
In this section, we include various experimental results dealing with training \ranem{} models.
In \autoref{tab:nem_training_time}, we explore the upfront training time needed for using \ranem{} and compare it to \nemt{}. We can see that \ranem{} takes longer to train, which is due to the additional samples used to increase the training stability of \ranem{}.
\begin{table}[H]
\centering
\caption{Time spent training NEM models for different explained models. The models are trained for 10 epochs on 10000 images extracted from the validation split of the ImageNet Model.}
\vspace{5pt}
\label{tab:nem_training_time}
% \usepackage{booktabs}
\begin{tabular}{@{}lrr@{}}
\toprule
\textbf{Method} & {\textbf{NEMt (sec.) $\downarrow$}} & {\textbf{Ra-NEM (sec.) $\downarrow$}} \\ \midrule
ResNet50& 335& 700\\
ConvNeXt& 360& 869\\
VGG16   & 437& 980\\
ViT     & 403& 913\\ \bottomrule
\end{tabular}
\end{table}

\ranem{} is a trained method; therefore, its performance can potentially vary between training runs. To explore this, we trained a \ranem{}  to explain a ResNet50 multiple times and explored the variability in faithfulness between runs. The training protocol was identical to the one used for the main results.  The experimental results can be seen in \autoref{tab:nem_training_stability}. We can see that performance was fairly stable across runs with a maximum difference in faithfulness of 0.021. 
\begin{table}[H]
\centering
\caption{Variability in faithfulness for a  Ra-NEM explaining a ResNet50  across five runs with the same hyperparameters. The results are very stable. }
\vspace{5pt}
\label{tab:nem_training_stability}
% \usepackage{booktabs}
\begin{tabular}{@{}lr@{}}
\toprule
{\textbf{Round}} & {\textbf{Faithfulness  $\uparrow$} } \\ \midrule
1                                  & 0.602                                     \\
2                                  & 0.581                                     \\
3                                  & 0.595                                     \\
4                                  & 0.597                                     \\
5                                  & 0.582                                     \\ \bottomrule
\end{tabular}

\end{table}

To understand the importance of choosing the number $k$ of samples used when training \ranem{}, we trained a \ranem{} to explain a ResNet50 model and varied the number of samples used throughout the runs. We subsequently measured training time and faithfulness to explore whether there exists a tradeoff between the two. The results are reported in \autoref{tab:num_samples}, where it can be seen that there is some trade-off between training time and faithfulness, which is controlled via $k$.

\begin{table}[H]
\centering
\caption{The effect of the number of samples on faithfulness and training time for the \ranem{} method when explaining a ResNet50 model. All hyperparameters except for number of samples are identical across runs. It can be seen that increasing $k$ can improve faithfulness but also increase training time. }
\vspace{5pt}
\label{tab:num_samples}
% \usepackage{booktabs}
\begin{tabular}{@{}lrr@{}}
\toprule
{\textbf{\textit{k}}} & \multicolumn{1}{c}{\textbf{Training time $\downarrow$}} & {\textbf{Faithfullness $\uparrow$}} \\ \midrule
1                              & 252                                        & 0.520                                      \\
2                              & 310                                        & 0.560                                  \\
4                              & 463                                        & 0.596                                      \\
6                              & 700                                        & 0.607                                    \\
8                              & 967                                        & 0.617                                      \\ \bottomrule
\end{tabular}
\end{table}

Finally, to explore the effect of the choice of perturbations during training \ranem{}, we compare the faithfulness of different \ranem{} models when using different perturbations. The results are shown in \autoref{tab:mr_ablation_training_perturbation_choice}, where it can be seen that the downstream faithfulness is fairly robust to the choice of perturbation, indicating that the overall method is robust to this hyperparameter choice

\begin{table}[ht!]
\centering
\caption{Impact of \ranem{} training perturbation choices on  validation performance measured by faithfulness (ConvNeXt).
}
\label{tab:mr_ablation_training_perturbation_choice}
%\resizebox{0.8\columnwidth}{!}
{%
\begin{tabular}{@{}lrrrr@{}}
\toprule
             & blur                     & noise                     & zero imputation & all \\ \midrule
Faithfulness & \multicolumn{1}{r}{0.450} & \multicolumn{1}{r}{0.465} & \multicolumn{1}{r}{0.464}        & \multicolumn{1}{r@{}}{0.468} \\ \bottomrule
\end{tabular}%
}
\end{table}

\section{Pseudocode for the omega operator}
\label{appendix:omega_pseudo_code}

The pseudocode in PyTorch style in \cref{alg:top-k_operator} shows how we implemented $\mask(x,\text{top}_k(\fa + g))$ in the loss function.
The arguments \texttt{A}, \texttt{x}, \texttt{k}, and \texttt{R}
refer to $a$, $x$, $k$, and the replacement values are denoted by $R\in\mathbb{R}^d$.
We use the ``Gumbel top-trick'' for top-$k$ feature selection. 
We treat the attribution $a$ as logits.
To ensure numerical stability, we normalize the attribution map using a logsumexp operation, which introduces a dependency between the individual feature attributions to prevent divergence.
After the noisy top-$k$ selection, 
we map the attributions to $[0,1]$ using a sigmoid centered on the $k$-th attribution value.  
After that, we apply a hard thresholding mapping the first $k$ attribution values to one and the others to zero. This step is ignored in the gradient computation using the straight-through operator \cite{bengio2013estimating}. This is why we perform soft thresholding using a sigmoid to guide the optimization. Essentially, we try to get the masked values close to either zero or one depending on its position relative to the $k$ largest feature value using the differentiable sigmoid function, before we perform the non-differentiable discretization step.
Finally, the resulting mask is applied to the original input. The masked-out features are replaced depending on which $\mask$ is used, which is determined by the additional $R$ parameter in \cref{alg:top-k_operator}.
For example, $R = 0$ yields $\mask = \maskdel$ and $R = \operatorname{blur}(x)$ corresponds to $\mask = \maskins$.

\begin{algorithm}[tb]
\caption{Omega operator in Python/PyTorch pseudocode.}
\label{alg:top-k_operator}
%\begin{center}
\begin{lstlisting}[style=python]
def omega_operator(A, x, k, R):
    # stabilize training
    A = A - A.logsumexp()
    # sample Gumbel noise
    U = torch.rand_like(A)
    G = -torch.log(-torch.log(U)) 
    # Gumbel perturbed sorting 
    A = A + G
    A, indices = torch.sort(A, descending=True)
    # soft thresholding
    A = A - A[k]
    A = torch.sigmoid(A)
    # hard thresholding
    mask = A - A.detach()
    # only keep top k elements
    mask[:k] += 1
    # reorder mask to original index order
    M = torch.zeros_like(A)
    M[indices] = mask
    # generate perturbed input
    x_M = x * M + R * (1 - M)  
    return x_M
\end{lstlisting}
%\end{center}
\end{algorithm}

\FloatBarrier
\section{Reproducibility check list}
\label{appendix:repro_check_list}
Here we provide an overview of the hyperparameter used during the experimental evaluations. 

%\textbf{Reproducibility checklist}
\begin{itemize}
  \item \emph{Temperature}: $1$ (fixed for the entire training)
  \item \emph{Number of samples ($k$)}: $6$
  \item \emph{Sampling strategy}: Divide the range $[1, d]$, where $d = 224 \times 224$ is the number of spatial dimensions, into $k$ uniform bins and sample one value from each bin
  \item \emph{Perturbation masks}: Uniformly sampled per batch from
    \begin{itemize}
      \item zero imputation
      \item noise imputation
      \item blur imputation
    \end{itemize}
  \item \emph{Noise imputation value}: ImageNet-scaled normal noise
  \item \emph{Blur imputation kernel}: Gaussian kernel (size $11$, $\sigma = 5$)
  \item \emph{Explanation post-processing}: Min–max normalization
\end{itemize}

\section{Example images for all models}
In this section, we show additional randomly selected example images for all models. We refer to the image captions for a qualitative discussion of the results.
\label{appendix:model_images}
%\subsection{ResNet50}
\label{sub_sec:img_resnet50}
\begin{figure*}[ht]
    \centering
    \includegraphics[width=\linewidth]{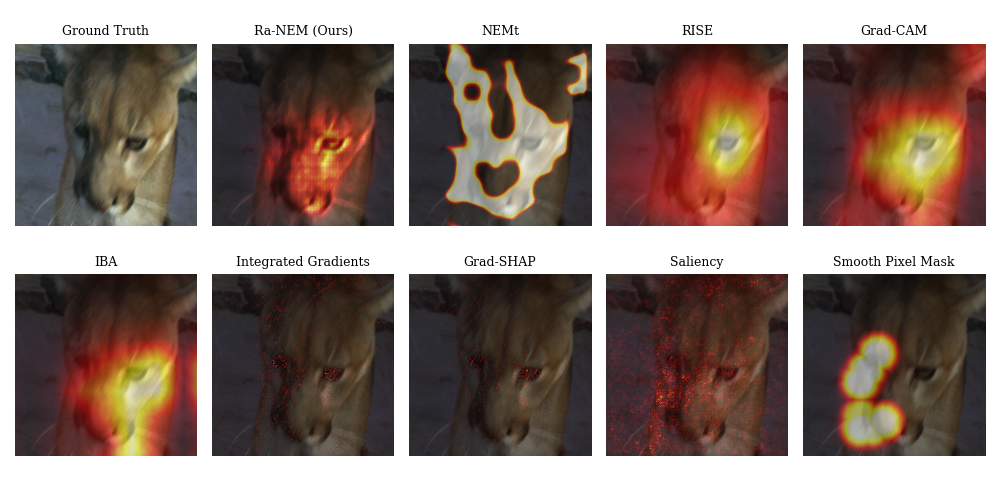}
    \vspace{-25pt}
    \caption{Model explanations generated by various different XAI methodologies. The explained model is a ResNet50 architecture trained on the training split of the ImageNet dataset. The example image  is taken from the validation split of the ImageNet dataset. Comparing \ranem{} with other occlusion-based methods, we observe that its finer granularity more precisely reveals specific image features. While all methods highlight the importance of the eye, \ranem{} identifies only a small region of the eye and its outline as necessary.}
    \label{fig:resnet501}
\end{figure*}
\begin{figure*}[ht]
    \centering
    \includegraphics[width=\linewidth]{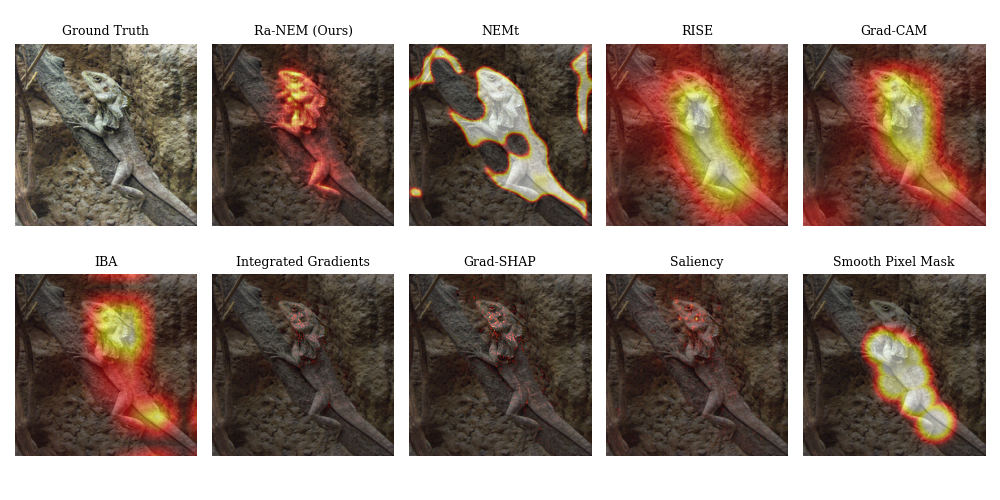}
    \vspace{-25pt}
    \caption{Model explanations generated by various different XAI methodologies. The explained model is a ResNet50 architecture trained on the training split of the ImageNet dataset. The example image  is taken from the validation split of the ImageNet dataset.}
    \label{fig:resnet502}
\end{figure*}
%\subsection{ConvNeXt}
\label{sub_sec:img_convnext}
\begin{figure*}[ht]
    \centering
    \includegraphics[width=\linewidth]{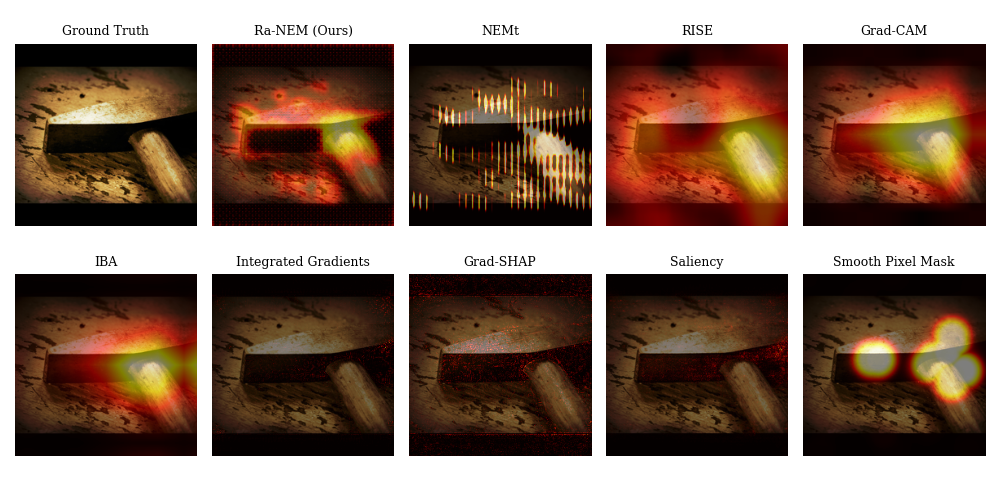}
    \vspace{-25pt}
    \caption{Model explanations generated by different XAI methodologies. The explained model is a ConvNeXt architecture trained on the training split of the ImageNet dataset. The example image  is taken from the validation split of the ImageNet dataset.  Occlusion-based methods generally agree on the area of interest, but only \ranem{} with its high granularity can outline specific features. For example, it highlights both the shaft of the hammer and the outline of the hammerhead, which indicates that most of the hammerhead is not needed for classification.}
    \label{fig:convnext0}
\end{figure*}
\begin{figure*}[ht]
    \centering
    \includegraphics[width=\linewidth]{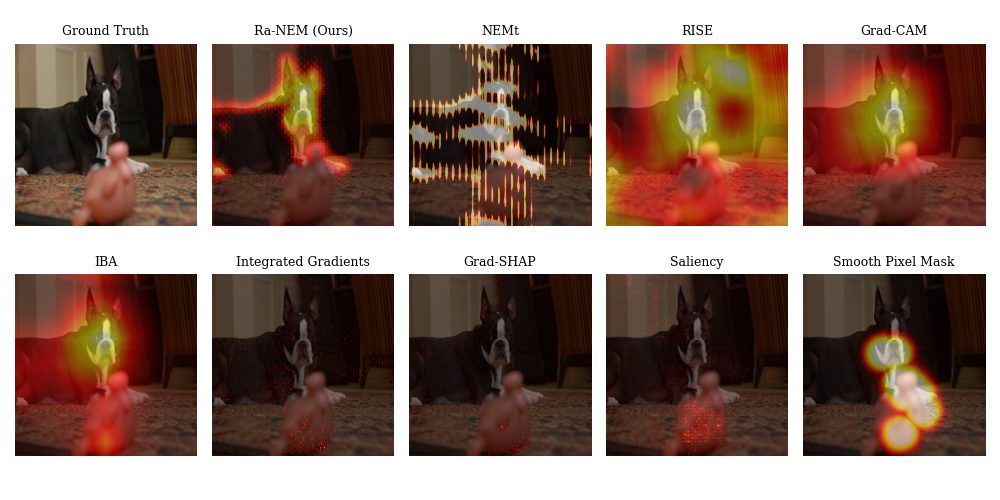}
    \vspace{-25pt}
    \caption{Model explanations generated by  different XAI methodologies. The explained model is a ConvNeXt architecture trained on the training split of the ImageNet dataset. The example image  is taken from the validation split of the ImageNet dataset. It can be seen that all occlusion-based methods agree that the face of the dog is important, but \ranem{} also emphasizes the ears.}
    \label{fig:convnext1}
\end{figure*}
\begin{figure*}[ht]
    \centering
    \includegraphics[width=\linewidth]{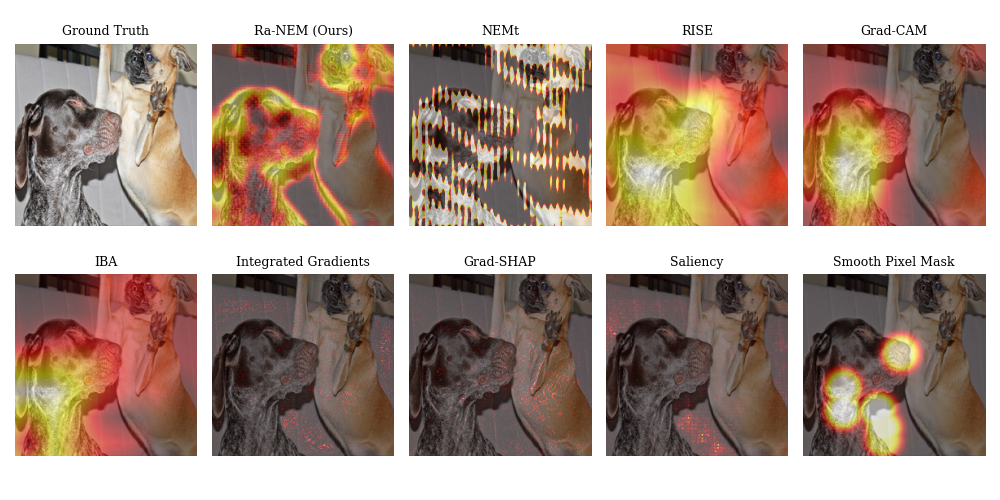}
    \vspace{-25pt}
    \caption{Model explanations generated by  different XAI methodologies. The explained model is a ConvNeXt architecture trained on the training split of the ImageNet dataset. The example image  is taken from the validation split of the ImageNet dataset. This example illustrates \ranem{}'s ability to discover fine details and as such indicate that the  outline of the dog is central to classification, whereas other methods biased toward smooth explanations need to indicate the entire animal to be important.  }
    \label{fig:convnext2}
\end{figure*}
%\subsection{VGG16}
\label{sub_sec:img_vgg16}
\begin{figure*}[ht]
    \centering
    \includegraphics[width=\linewidth]{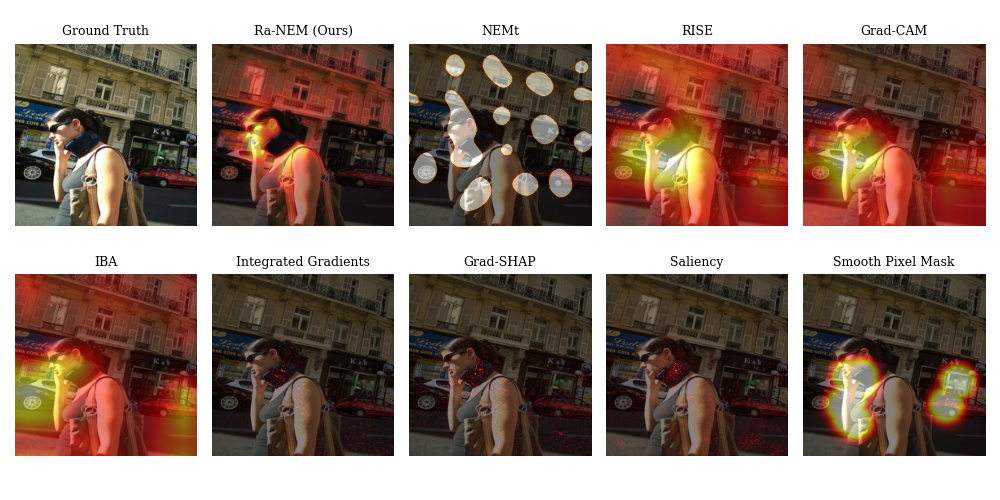}
    \vspace{-25pt}
    \caption{Model explanations generated by  different XAI methodologies. The explained model is a VGG16 architecture trained on the training split of the ImageNet dataset. The example image  is taken from the validation split of the ImageNet dataset. Notably, \ranem{} excludes the neck brace, whereas the smoothing of the other occlusion based methods somewhat include it as an important feature. Furthermore, it can be seen that the gradient based methods all generally target the neck brace.}
    \label{fig:vgg160}
\end{figure*}
\begin{figure*}[ht]
    \centering
    \includegraphics[width=\linewidth]{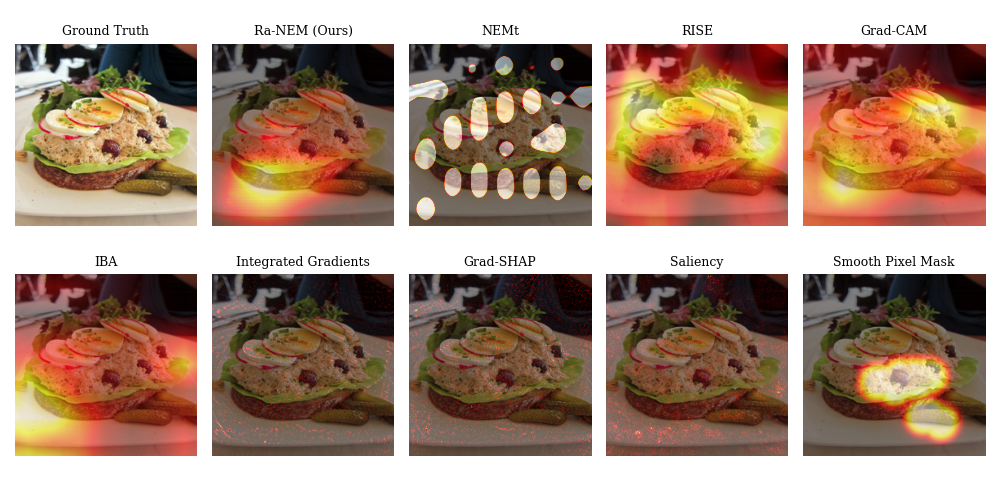}
    \vspace{-25pt}
    \caption{Model explanations generated by  different XAI methodologies. The explained model is a VGG16 architecture trained on the training split of the ImageNet dataset. The example image  is taken from the validation split of the ImageNet dataset. }
    \label{fig:vgg161}
\end{figure*}
\begin{figure*}[ht]
    \centering
    \includegraphics[width=\linewidth]{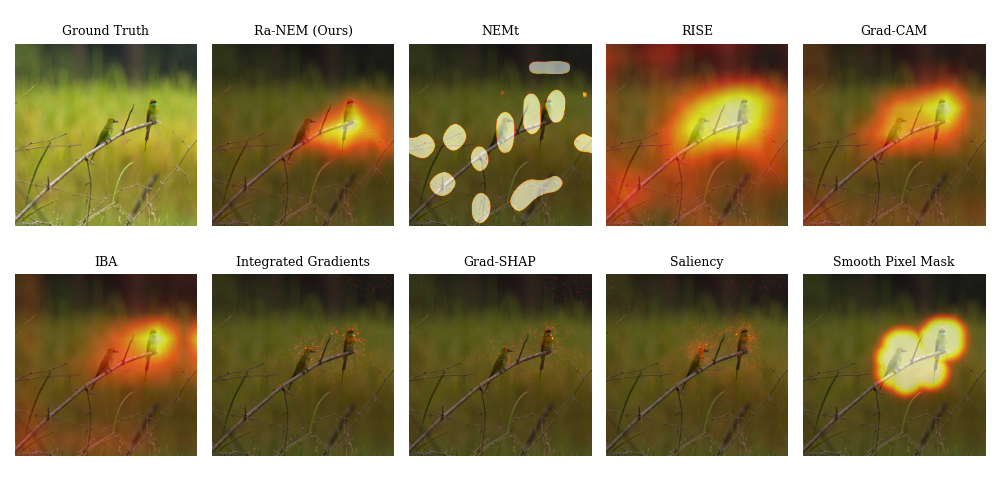}
    \vspace{-25pt}
    \caption{Model explanations generated by  different XAI methodologies. The explained model is a VGG16 architecture trained on the training split of the ImageNet dataset. The example image  is taken from the validation split of the ImageNet dataset. }
    \label{fig:vgg162}
\end{figure*}
\newpage
%\subsection{ViT}
\label{sub_sec:vit_images}
\begin{figure*}[ht]
    \centering
    \includegraphics[width=\linewidth]{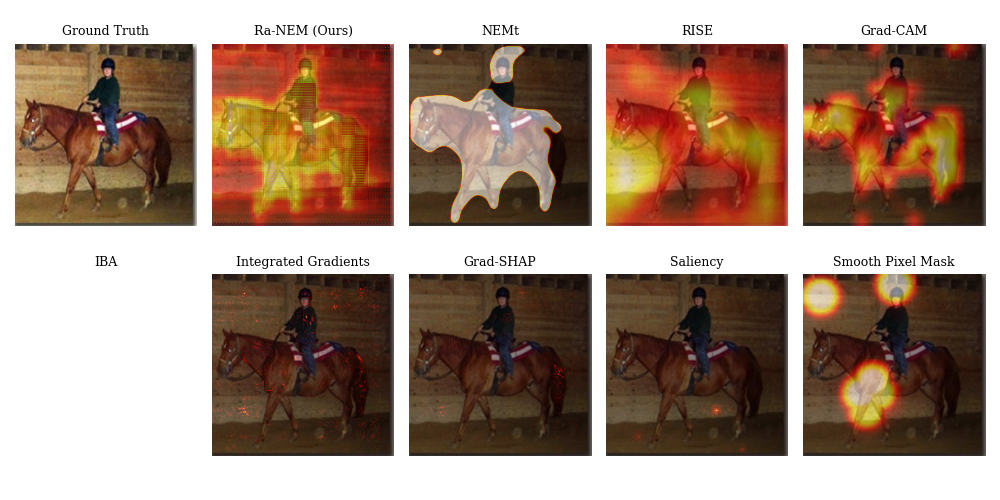}
    \vspace{-25pt}
    \caption{Model explanations generated by different XAI methodologies. The explained model is a ViT architecture trained on the training split of the ImageNet dataset. The example image  is taken from the validation split of the ImageNet dataset. There is no adaptation of IBA for transformer-based image models, so IBA results have not been generated. This exemplary image illustrates a number of important points. We can see that the  methods which leverage smoothing (\gradcam{} and \rise{}) generate attributions that are very scattered over the input, which might be due to an unfortunate combination of the smoothing and how the ViT processes an image. Given the ViT relies on attention mechanisms instead of convolutions, it could be that the architecture's bias toward spatial cohesion is much lower and therefore enforcing a spatial bias in the attributions might have an adverse effect. Additionally, while \intgrad{} may achieve high faithfulness, its attributions offer limited interpretability for end users. In contrast, \ranem{} provides a balanced trade-off between interpretability and faithfulness.}
    \label{fig:vit2}
\end{figure*}

\begin{figure*}[ht]
    \centering
    \includegraphics[width=\linewidth]{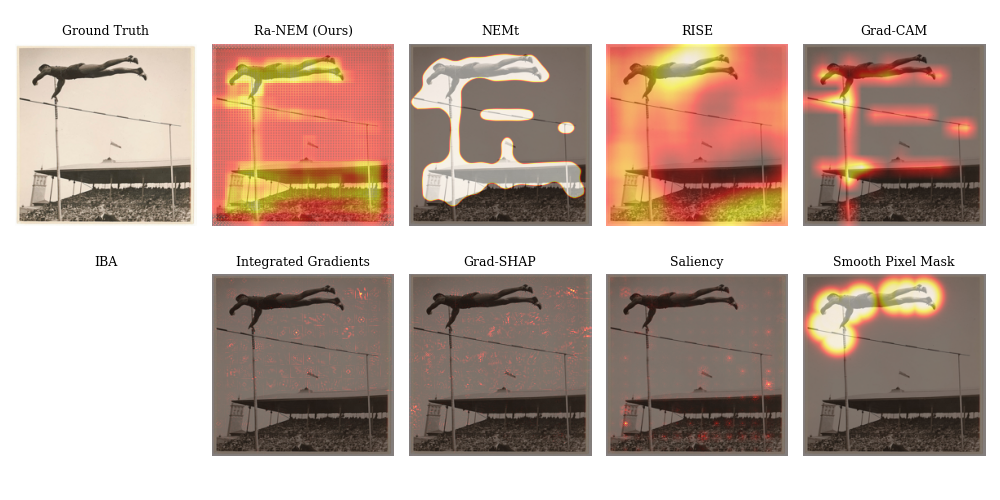}
    \vspace{-25pt}
    \caption{Model explanations generated by different XAI methodologies. The explained model is a ViT architecture trained on the training split of the ImageNet dataset. The example image  is taken from the validation split of the ImageNet dataset. There is no adaptation of IBA for transformer-based image models, so IBA results have not been generated. }
    \label{fig:vit0}
\end{figure*}

\begin{figure*}[ht]
    \centering
    \includegraphics[width=\linewidth]{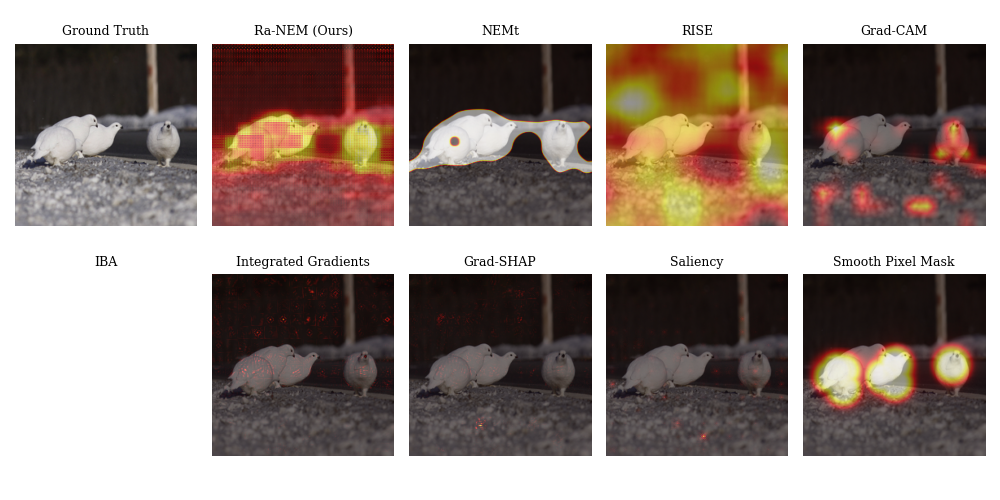}
    \vspace{-25pt}
    \caption{Model explanations generated by different XAI methodologies. The explained model is a ViT architecture trained on the training split of the ImageNet dataset. The example image  is taken from the validation split of the ImageNet dataset. As there is no adaptation of IBA for transformer-based image models, IBA results have not been generated. This example underpins many of the points already discussed in \autoref{fig:vit2}, a.i.{} the confusion of  methods leveraging smoothing and the limited interpretability of \intgrad{}, despite its higher faithfulness score. Additionally, \nemt{} occasionally generates very large attributions for the ViT.}
    \label{fig:vit1}
\end{figure*}

%\FloatBarrier
\section{Example images of poor \ranem{} performance}
\label{appendix:poor_performance}
In the following, we show examples where \ranem{} performed poorly compared to other methods to understand the limitations of the method.  These examples were found by selecting the images for which the faithfulness scores of \ranem{} were the lowest compared to the average faithfulness of all other methods.  

\begin{figure*}[ht]
    \centering
    \includegraphics[width=\linewidth]{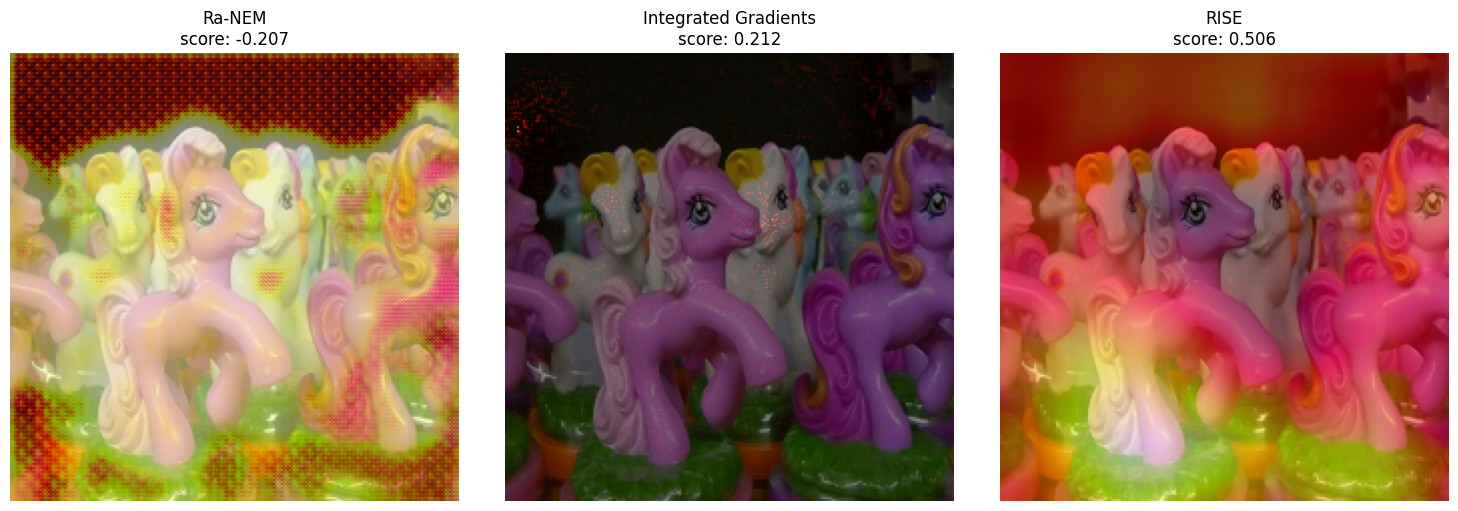}

    \caption{Example of image from ImageNet test set, where \ranem{} performs relatively poor compared to other methods. Score is Faithfulness. We include \intgrad{} and \rise{} for comparison. The explained model is a ConVNeXt trained on the training split of the ImageNet dataset.}
    \label{fig:bad_convnext1}
\end{figure*}

\begin{figure*}[ht]
    \centering
    \includegraphics[width=\linewidth]{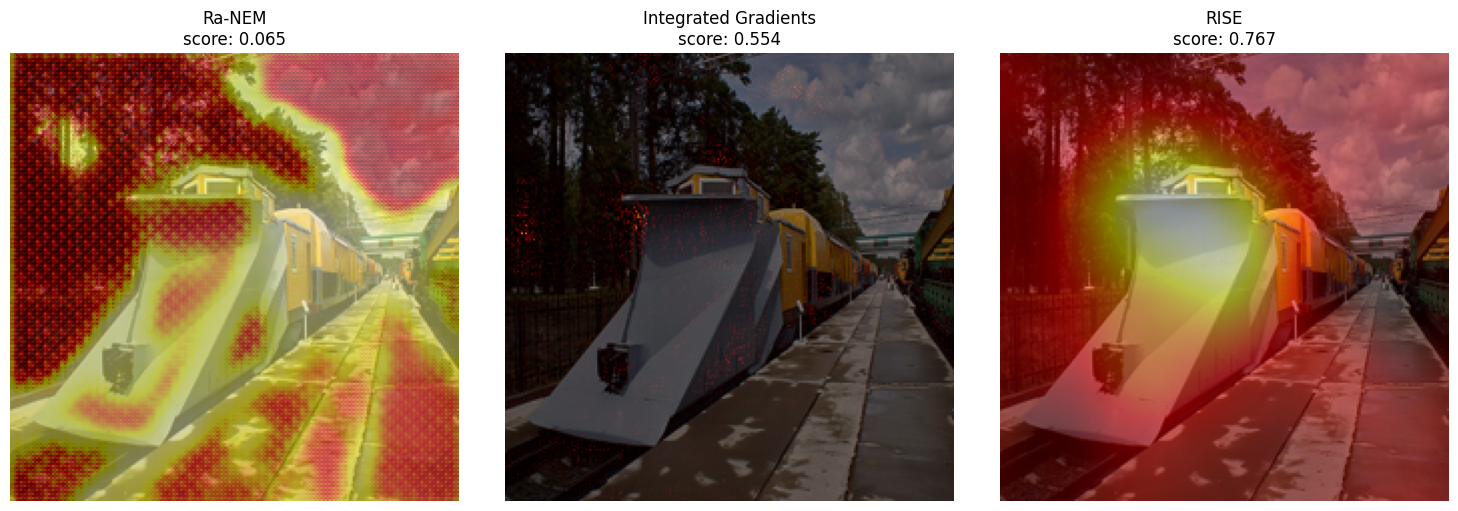}
    \caption{Example of image from ImageNet test set, where \ranem{} performs relatively poor compared to other methods. Score is Faithfulness. We include \intgrad{} and \rise{} for comparison. The explained model is a ConVNeXt trained on the trainings split of the ImageNet dataset. }
    \label{fig:bad_convnext2}
\end{figure*}

\begin{figure*}[ht]
    \centering
    \includegraphics[width=\linewidth]{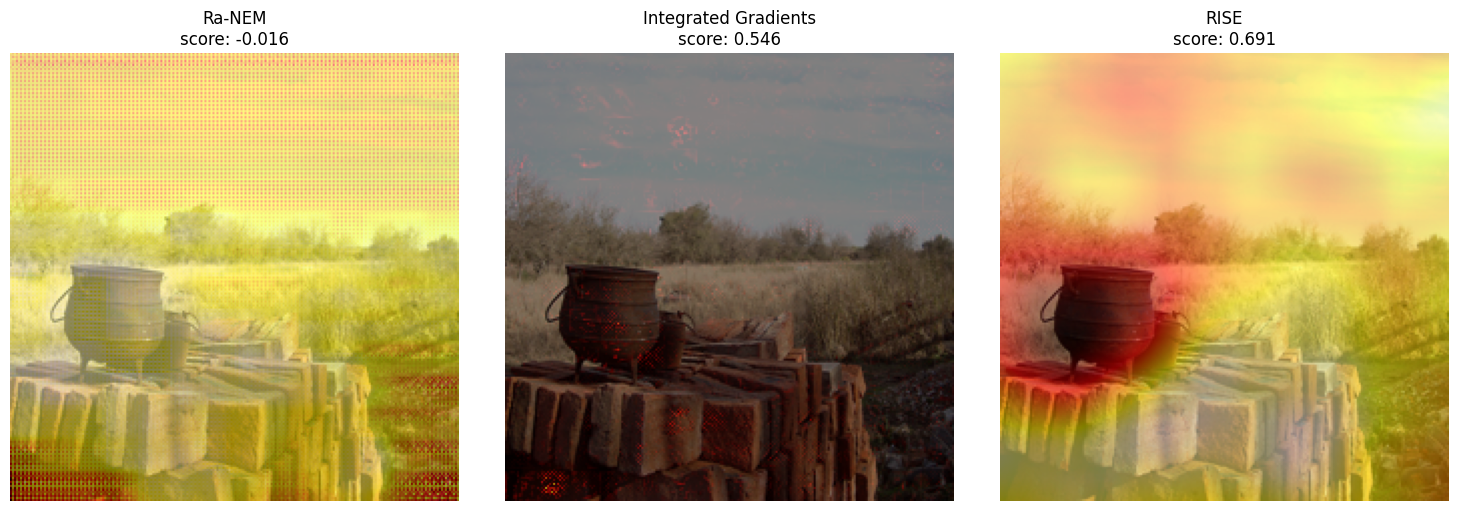}
    \caption{Example of image from ImageNet test set, where \ranem{} performs relatively poor compared to other methods. Score is Faithfulness. We include \intgrad{} and \rise{} for comparison. The explained model is a ViT trained on the trainings split of the ImageNet dataset. Here \ranem{} appears to not properly distinguish between different elements in the images.}
    \label{fig:bad_vit1}
\end{figure*}

\begin{figure*}[ht]
    \centering
    \includegraphics[width=\linewidth]{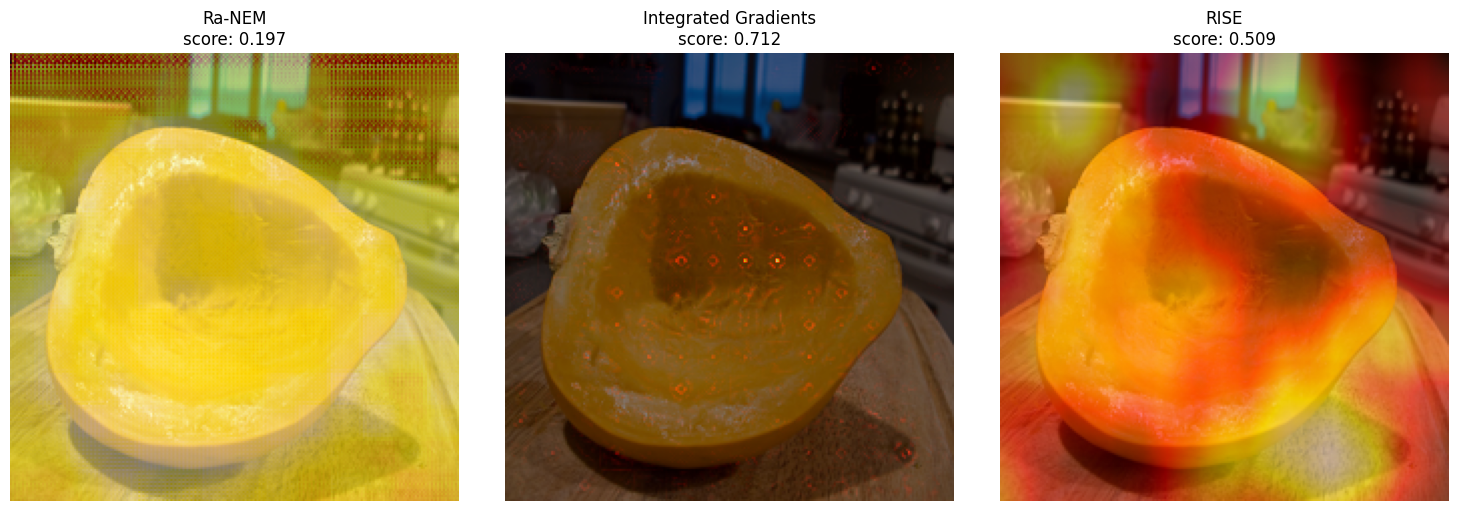}
    \caption{Example of image from ImageNet test set, where \ranem{} performs relatively poor compared to other methods. Score is Faithfulness. We include\intgrad{} and \rise{} for comparison. The explained model is a ViT trained on the trainings split of the ImageNet dataset. Here \ranem{} again does not appear to properly distinguish between objects in the image.}
    \label{fig:bad_cvit2}
\end{figure*}

\begin{figure*}[ht]
    \centering
    \includegraphics[width=\linewidth]{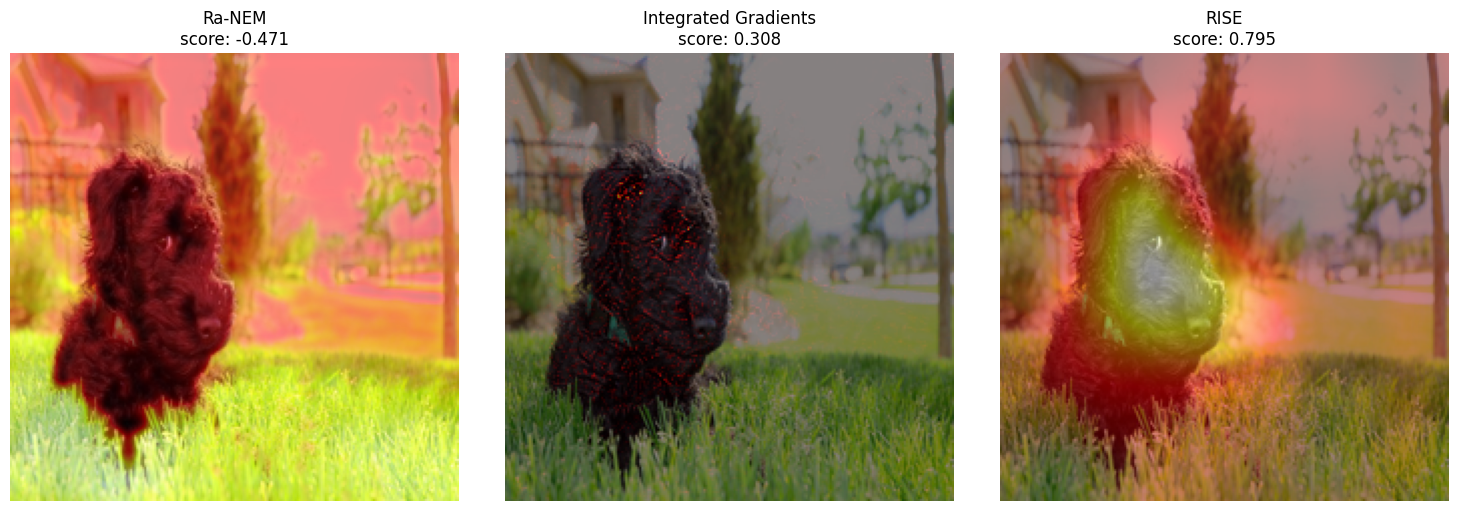}
    \caption{Example of image from ImageNet test set, where \ranem{} performs relatively poor compared to other methods. Score is Faithfulness. We include \intgrad{} and \rise{} for comparison. The explained model is a VGG16 trained on the trainings split of the ImageNet dataset.  Here it can be seen, that \ranem{} has switched foreground for background.}
    \label{fig:bad_vgg1}
\end{figure*}
\begin{figure*}[ht]
    \centering
    \includegraphics[width=\linewidth]{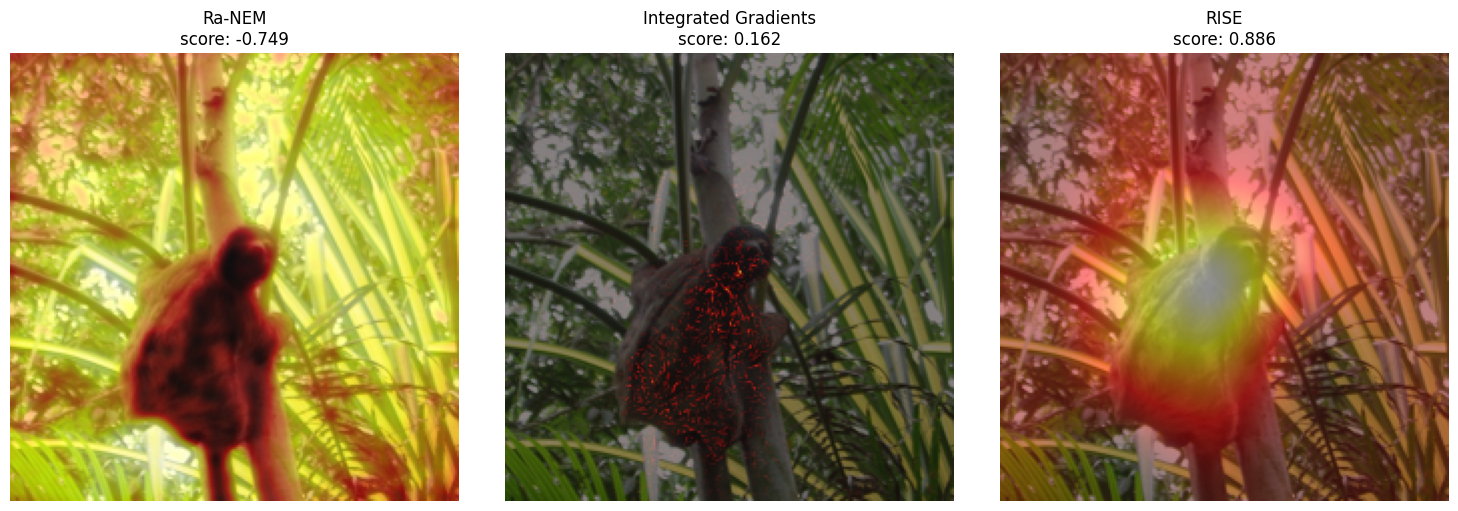}
    
    \caption{Example of image from ImageNet test set, where \ranem{} performs relatively poor compared to other methods. Score is Faithfulness. We include \intgrad{} and \rise{} for comparison. The explained model is a VGG16 trained on the trainings split of the ImageNet dataset. Here it can be seen, that \ranem{} has again switched foreground for background. }
    \label{fig:bad_vgg2}
\end{figure*}

\begin{figure*}[ht]
    \centering
    \includegraphics[width=\linewidth]{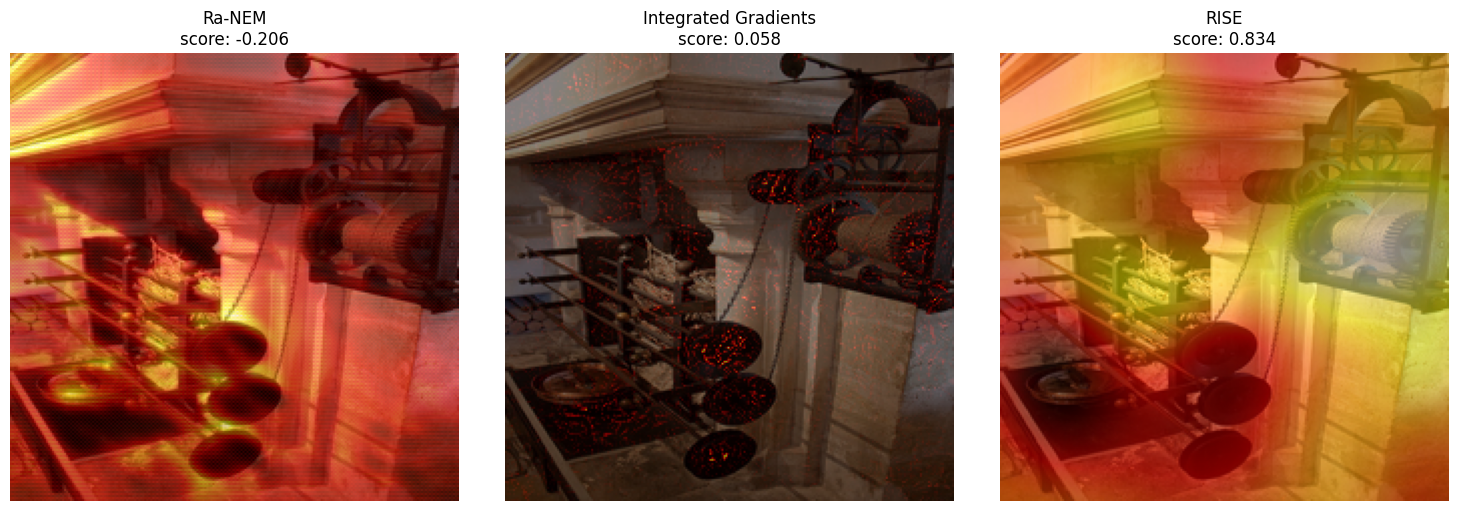}
    \caption{Example of image from ImageNet test set, where \ranem{} performs relatively poor compared to other methods. Score is Faithfulness. We include \intgrad{} and \rise{} for comparison. The explained model is a ResNet50 trained on the trainings split of the ImageNet dataset. Here \ranem{} predicts that the wrong object is the most important.}
    \label{fig:bad_resnet1}
\end{figure*}
\begin{figure*}[ht]
    \centering
    \includegraphics[width=\linewidth]{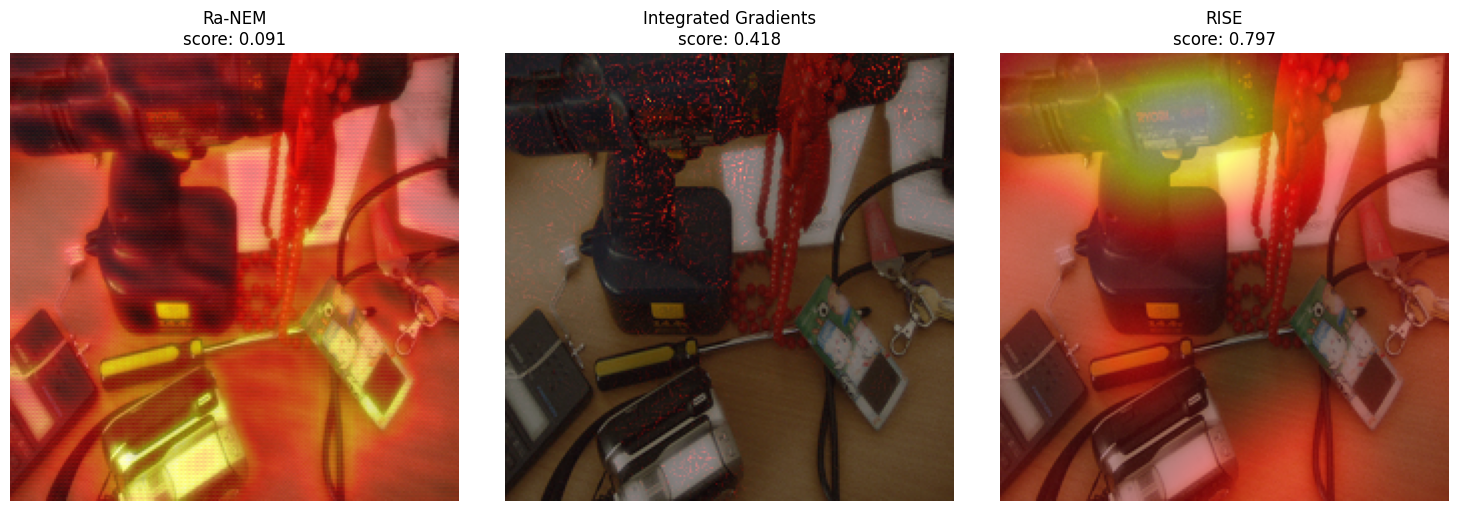}
    \caption{Example of image from ImageNet test set, where \ranem{} performs relatively poor compared to other methods. Score is Faithfulness. We include \intgrad{} and \rise{} for comparison. The explained model is a ResNet50 trained on the trainings split of the ImageNet dataset. Here \ranem{} again focuses on the wrong object in the image.}
    \label{fig:bad_resnet2}
\end{figure*}

\end{document}